%% file: AISTATS_2025.tex
\documentclass[twoside]{article}
\usepackage{fullpage}
\usepackage{microtype}
\usepackage{graphicx}
\usepackage{subfigure}
\usepackage{amsmath}
\usepackage{amssymb}
\usepackage{mathtools}
\usepackage{amsthm}
\usepackage[scr=boondox]{mathalpha}

\usepackage[accepted]{aistats2025}
\usepackage[round]{natbib}

\usepackage{diagbox}
\usepackage[utf8]{inputenc} 
\usepackage[T1]{fontenc}    
\usepackage{hyperref}       
\usepackage{url}            
\usepackage{booktabs}       
\usepackage{amsfonts}       
\usepackage{nicefrac}       
\usepackage{microtype}      
\usepackage{xcolor,enumitem}         

\theoremstyle{plain}
\newtheorem{theorem}{Theorem}
\newtheorem{proposition}[theorem]{Proposition}
\newtheorem{lemma}{Lemma}

\theoremstyle{definition}

\newtheorem{assumption}[theorem]{Assumption}
\newtheorem{remark}[theorem]{Remark}

\begin{document}

%

%

\twocolumn[

\aistatstitle{Gaussian Smoothing in Saliency Maps: The Stability-Fidelity Trade-Off in Neural Network Interpretability 
}

\aistatsauthor{ Zhuorui Ye \And Farzan Farnia}

\aistatsaddress{ Tsinghua University \And The Chinese University of Hong Kong} ]

\begin{abstract}
\input{0-abstract}
\end{abstract}

\section{INTRODUCTION}
\input{1-Introduction}

\section{RELATED WORK}
\input{2-RelatedWork}

\section{PRELIMINARIES}
\input{2b-prelim}

\section{ALGORITHMIC STABILITY OF SALIENCY MAPS}
\input{3-Stability-of-Maps}

\subsection{Fidelity of Gaussian Smoothed Saliency Maps}
\input{4-FideltyMaps}

\section{NUMERICAL EXPERIMENTS}

\input{5-NumericalResults}

\section{CONCLUSION}
\input{6-Conclusion}

\clearpage
\clearpage
\section*{Acknowledgments}
 
The work of Farzan Farnia is partially supported by a grant from the Research Grants Council of the Hong Kong Special Administrative Region, China, Project 14209920, and is partially supported by CUHK Direct Research Grants with CUHK Project No. 4055164 and 4937054. 
Also, the authors would like to thank the anonymous reviewers for their constructive feedback and suggestions.
{
\bibliography{ref}
}

\section*{Checklist}



 \begin{enumerate}

 \item For all models and algorithms presented, check if you include:
 \begin{enumerate}
   \item A clear description of the mathematical setting, assumptions, algorithm, and/or model. [Yes]
   \item An analysis of the properties and complexity (time, space, sample size) of any algorithm. [Yes]
   \item (Optional) Anonymized source code, with specification of all dependencies, including external libraries. [No]
 \end{enumerate}

 \item For any theoretical claim, check if you include:
 \begin{enumerate}
   \item Statements of the full set of assumptions of all theoretical results. [Yes]
   \item Complete proofs of all theoretical results. [Yes]
   \item Clear explanations of any assumptions. [Yes]     
 \end{enumerate}

 \item For all figures and tables that present empirical results, check if you include:
 \begin{enumerate}
   \item The code, data, and instructions needed to reproduce the main experimental results (either in the supplemental material or as a URL). [Yes]
   \item All the training details (e.g., data splits, hyperparameters, how they were chosen). [Yes]
         \item A clear definition of the specific measure or statistics and error bars (e.g., with respect to the random seed after running experiments multiple times). [Yes]
         \item A description of the computing infrastructure used. (e.g., type of GPUs, internal cluster, or cloud provider). [Yes]
 \end{enumerate}

 \item If you are using existing assets (e.g., code, data, models) or curating/releasing new assets, check if you include:
 \begin{enumerate}
   \item Citations of the creator If your work uses existing assets. [Yes]
   \item The license information of the assets, if applicable. [Not Applicable]
   \item New assets either in the supplemental material or as a URL, if applicable. [Not Applicable]
   \item Information about consent from data providers/curators. [Not Applicable]
   \item Discussion of sensible content if applicable, e.g., personally identifiable information or offensive content. [Not Applicable]
 \end{enumerate}

 \item If you used crowdsourcing or conducted research with human subjects, check if you include:
 \begin{enumerate}
   \item The full text of instructions given to participants and screenshots. [Not Applicable]
   \item Descriptions of potential participant risks, with links to Institutional Review Board (IRB) approvals if applicable. [Not Applicable]
   \item The estimated hourly wage paid to participants and the total amount spent on participant compensation. [Not Applicable]
 \end{enumerate}

 \end{enumerate}

\newpage
\appendix
\renewcommand{\thetheorem}{\Alph{section}.\arabic{theorem}}
\setcounter{theorem}{0}
\renewcommand{\thelemma}{\Alph{section}.\arabic{lemma}}
\setcounter{lemma}{0}
\onecolumn
\input{7-Appendix}

\end{document}

%% file: 0-abstract.tex
Saliency maps have been widely used to interpret the decisions of neural network classifiers and discover phenomena from their learned functions. However, standard gradient-based maps are frequently observed to be highly sensitive to the randomness of training data and the stochasticity in the training process. In this work, we study the role of Gaussian smoothing in the well-known Smooth-Grad algorithm in the stability of the gradient-based maps to the randomness of training samples. We extend the algorithmic stability framework to gradient-based interpretation maps and prove bounds on the stability error of standard Simple-Grad, Integrated-Gradients, and Smooth-Grad saliency maps. Our theoretical results suggest the role of Gaussian smoothing in boosting the stability of gradient-based maps to the randomness of training settings. On the other hand, we analyze the faithfulness of the Smooth-Grad maps to the original Simple-Grad and show the lower fidelity under a more intense Gaussian smoothing. We support our theoretical results by performing several numerical experiments on standard image datasets. Our empirical results confirm our hypothesis on the fidelity-stability trade-off in the application of Gaussian smoothing to gradient-based interpretation maps. 

%% file: 1-Introduction.tex
\begin{figure}[t]
  \centering
  \includegraphics[width=0.99\columnwidth]{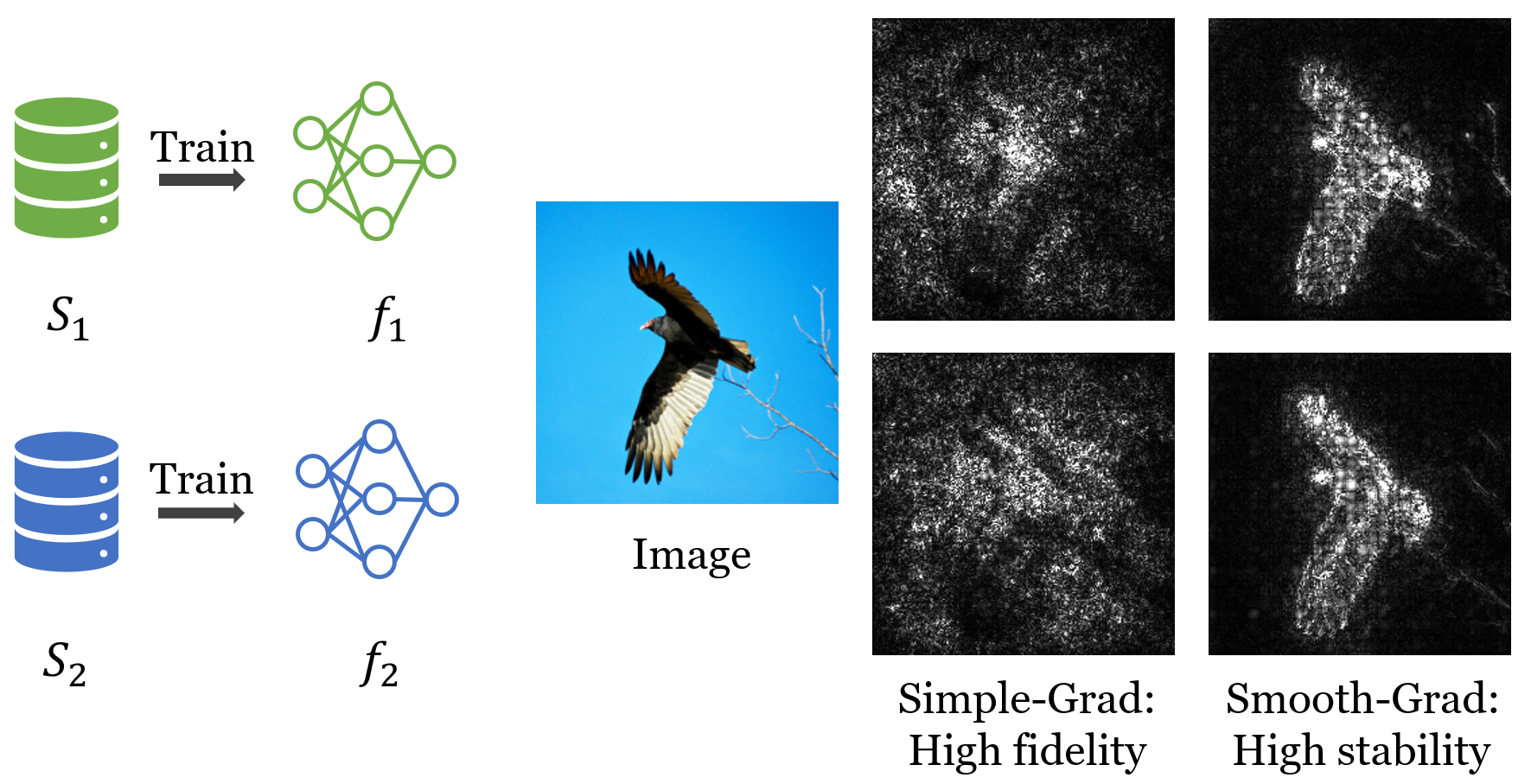}
  \vspace{-0.4cm}
  \caption{The stability-fidelity trade-off introduced by Gaussian smoothing. We trained neural networks on two disjoint training set splits of ImageNet, and computed the Simple-Grad and Smooth-Grad maps for the same test sample. 
  }
  \label{fig:tradeoff}
\vspace{-0.1cm}
\end{figure}

Deep learning models have attained state-of-the-art results over a wide array of supervised learning tasks including image classification \citep{krizhevsky2012imagenet}, speech recognition \citep{graves2013speech}, and text categorization \citep{minaee2021deep}. The trained deep neural networks have been utilized not only to address the target classification task but also to discover the underlying rules influencing the label assignment to a feature vector. To this end, saliency maps, which highlight the features influencing the neural network's predictions, have been widely used to gain an understanding of phenomena from data-driven models~\citep{freiesleben2022scientific} and further to find new discoveries and insights. For example, \citet{mitani2020detection} shows the application of saliency maps to identify the region in fundus images related to anemia, and \citet{bien2018deep} demonstrates that saliency maps can assist clinicians in diagnosing knee injury from MRI scans.

Specifically, gradient-based maps have been widely applied to compute saliency maps for neural net classifiers. Standard gradient-based saliency maps such as Simple-Grad \citep{simonyan2013deep} and Integrated- Gradients \citep{sundararajan2017axiomatic}, represent the input-based gradient of a neural network at a test sample, which reveals the input features with a higher local impact on the neural network classifier's output. Therefore, the assignment of feature importance scores by a gradient-based saliency map can be utilized to reveal the main features influencing the phenomenon. 

However, a consideration for drawing conclusions from gradient-based maps is their potential sensitivity to the training algorithm of the neural network classifier. \citet{arun2021assessing,woerl2023initialization} have shown that gradient maps could be significantly different for independent yet identically distributed training sets, or for a different random initialization used for training the classifier. The sensitivity of the saliency maps is indeed valuable in debugging applications, where the map is used to identify potential reasons behind a misclassification case.  
On the other hand, the instability of gradient-based maps could hinder the applications of Simple-Grad and Integrated-Gradients maps for phenomena discovery~\citep{arun2021assessing}, as for a generalizable understanding of the underlying phenomenon, the resulting saliency maps need to have limited dependence on the stochasticity of training data sampled from the underlying data distribution. 

In this work, we specifically study the influence of Gaussian smoothing used in the Smooth-Grad algorithm~\citep{smilkov2017smoothgrad} on the stability of a gradient-based map. We provide theoretical and numerical evidence that Gaussian smoothing could increase the stability of saliency maps to the stochasticity of the training setting, which can improve the reliability of applying the gradient-based map for phenomena understanding using the neural net classifier. 

To theoretically analyze the stability properties of saliency maps, we utilize the algorithmic stability framework in \citet{bousquet2002stability} to quantify the expected robustness of gradient maps to the stochasticity of a randomized training algorithm, e.g. stochastic gradient descent (SGD), and randomness of training data drawn from the underlying distribution. Following the analysis in \citet{hardt2016train}, we define the stability error of a stochastic training algorithm and bound the error in terms of the number of training data and SGD training iterations for the Simple-Grad, Integrated-Gradients, and Smooth-Grad. Our stability error bounds suggest the improvement in stability by applying the randomized smoothing mechanism in Smooth-Grad. Specifically, our error bounds improve by a factor $1/\sigma$ under a standard deviation parameter $\sigma$ of the Smooth-Grad's Gaussian noise. 

In addition to the stability of gradient maps under Gaussian smoothing, we also bound the faithfulness of the Gaussian smoothed saliency maps to the original Simple-Grad and Integrated-Gradients interpretation maps, for which we define fidelity error\footnote{In this paper, the fidelity term refers to the faithfulness of the regularized saliency map to the original saliency map.}. We show that by increasing the standard deviation parameter $\sigma$, the fidelity error grows proportionally to parameter $\sigma$. This relationship indicates the stabilization offered by Gaussian smoothing at the price of a higher fidelity error compared to the original saliency map. As illustrated in Figure~\ref{fig:tradeoff}, the Smooth-Grad saliency maps could be considerably more stable to the training set of the neural network, while they could significantly differ from the original Simple-Grad map. We note that different applications of saliency maps may prioritize stability or fidelity differently. For instance, debugging misclassifications would require higher fidelity, while phenomena discovery would prioritize higher stability. Therefore, understanding the stability-fidelity properties of Smooth-Grad maps helps with a principled application of the algorithm to different tasks. 

We perform numerical experiments to test our theoretical results on the algorithmic stability and fidelity of interpretation maps under Gaussian smoothing. In the experiments, we tested several standard saliency maps and image datasets. We empirically analyzed a broader range of instability sources in the training setting and demonstrated that Gaussian smoothing can lead to higher stability to the changes in the training algorithm and neural network architecture. The numerical results indicate the impact of Gaussian smoothing in reducing the sensitivity of the gradient-based interpretation map to the stochasticity of the training algorithm at the cost of a higher difference from the original gradient map. We can summarize this work's main contribution as:
\begin{itemize}[leftmargin=*]
    \item Studying the algorithmic stability and generalization properties of gradient-based saliency maps.
    \item Proving bounds on the algorithmic stability error of saliency maps under vanilla and noisy stochastic gradient descent (SGD) training algorithms.
    \item Analyzing the fidelity of Smooth-Grad maps and their faithfulness to the Simple-Grad map.
    \item Providing numerical evidence on the stability-fidelity trade-off of Smooth-Grad maps on standard image datasets. 
\end{itemize}

%% file: 2-RelatedWork.tex
\noindent \textbf{Gradient-based Interpretation} A prevalent method for generating saliency maps involves computing the gradient of a deep neural network's output with respect to an input image. This technique is extensively utilized in numerous related studies, including Smooth-Grad~\citep{smilkov2017smoothgrad}, Integrated-Gradients~\citep{sundararajan2017axiomatic}, DeepLIFT~\citep{shrikumar2017learning}, Grad-CAM~\citep{selvaraju2017grad}, and Grad-CAM++~\citep{chattopadhay2018generalized}. However, gradient-based saliency maps often exhibit considerable noise. To address this, various methods have been proposed to enhance the quality of these maps by reducing noise. Common strategies include altering the gradient flow through activation functions~\citep{zeiler2014visualizing,springenberg2014striving}, eliminating negative or minor activations~\citep{springenberg2014striving,kim2019saliency,ismail2021improving}, and integrating sparsity priors~\citep{levine2019certifiably,zhang2023moreaugrad}. Despite these advancements, these techniques do not explicitly tackle the noise and variability in stochastic optimization and the randomness of training data.
 
\noindent \textbf{Stability Analysis for Interpretation Maps} Beyond visual quality, several studies have explored the stability of interpretation methods. \citet{arun2021assessing} examined the consistency within the same architecture and across different architectures for various interpretation methods on medical imaging datasets, discovering that most methods failed in their tests. \citet{woerl2023initialization} conducted experiments revealing that neural networks with different initializations produce distinct saliency maps. Similarly, \citet{fel2022good} highlighted the instability in interpretation maps and proposed a metric to evaluate the generalizability of these maps. To address this instability, \citet{woerl2023initialization} suggested a Bayesian marginalization technique to eliminate noise from random initialization and stochastic training, though it is computationally intensive due to the need to train multiple networks. Furthermore, \citet{zhang2023moreaugrad} propose MoreauGrad which is a robust and sparsified version of the SimpleGrad and SmoothGrad algorithms.
The related works \citep{gong2024structured,gong2025boosting} develop adversarial training methods for improving the robustness and visual quality of saliency maps, while \citet{gong2024super} leverage image super-pixels for enhanced stability. On the other hand, our work focuses on the stability effects offered by Gaussian smoothing applied in SmoothGrad. 



\textbf{Sanity checks for saliency maps} Evaluating the application of saliency maps has been studied in several related works. The related work \citep{adebayo2018sanity}
proposes sanity checks for interpretation maps, where the saliency map is supposed to depend on the characteristics of input data and how the label $y$ is determined by the input feature vector $\mathbf{x}$. 
We note that the algorithmic stability considered in our paper is orthogonal to the sanity checks designed in \citet{adebayo2018sanity}, as the algorithmic stability analysis is performed while leaving the distribution $p_{y|\mathbf{x}}$ of training data unchanged.  
Similarly, the algorithmic stability notion in our work is independent of the robustness of interpretation maps to adversarial perturbations studied by \citet{ghorbani2019interpretation} and the insensitivity of the interpretation maps to unrelated features discussed in \citet{kindermans2019reliability}, as our defined algorithmic stability implies neither robustness of the map to adversarially-designed perturbations nor its insensitivity to semantically unassociated features. 


%% file: 2b-prelim.tex
\subsection{Supervised Learning and Neural Network Optimization}
\label{subsec:preliminary}
In this work, we consider a classification task with a neural network classifier. The supervised learner has access to the training set $S=\{(x_i,y_i)\}_{i=1}^n$ containing $n$ samples, independently drawn from a population distribution $P_{X,Y}$. We use $m$ to denote the dimension of input feature vector $x\in \mathcal X\subset \mathbb R^m$. Also, $c$ denotes the number of classes in the classification task, i.e, $y\in\{1,2,\dots,c\}$. We apply a gradient-based training algorithm to learn the parameters of the neural network with the training set $S$. 


Our analysis focuses on a class of $k$-layer neural networks. The prediction function can be represented by $f_{\mathcal W}(x)=W_k\phi(W_{k-1}\phi(\dots\phi(W_1x)))\in \mathbb R^c$, with $\mathcal W$ denoting the vector containing all parameters. We assume $\phi(\cdot)$ is 1-Lipschitz and $\phi(0)=0$. 
The neural network parameters are learned by minimizing the loss function defined over the training set, which is \begin{align}
    \min_{\mathcal W}\frac{1}{n}\sum_{i=1}^n\ell(\mathcal W,x_i,y_i)
\end{align} where the loss function computes the difference between prediction logits and the ground-truth label. This problem formulation is well-known as Empirical Risk Minimization (ERM). 
We assume our network loss function which takes as input the logit and a class label $\ell(o,y):\mathbb R^c\times \mathcal Y\to \mathbb R$ to be 1-Lipschitz. The commonly used cross-entropy loss function satisfies this property.

To train the neural network's parameters, we consider the standard stochastic gradient descent (SGD) optimizer that performs $T$ iterations of uniformly selecting a training data point $(x_i,y_i)$ and using this rule: \begin{align}
    \mathcal W_{t+1}=\mathcal W_{t}-\alpha_t\nabla_{\mathcal W}\ell(\mathcal W_t,x_i,y_i)\label{SGD}
\end{align}

We also consider the noisy stochastic gradient descent (SGD) algorithm, with the following update rule at iteration $t$: 
\begin{align}
    &\mathcal W_{t+1}=\,\mathcal W_{t}-\alpha_t\nabla_{\mathcal W}\widetilde{\ell}(\mathcal W_t,x_i,y_i)\nonumber \\
    \text{where }\;\; &\widetilde{\ell}(\mathcal W,x,y):=\,\mathbb{E}_{\mathcal V\in N(0,\kappa^2I)}[\ell(\mathcal W+\mathcal V,x,y)]\label{smooth-SGD}
\end{align}
In the following proposition, we show the noisy loss function optimized in the above update rule is a smooth function.
\begin{proposition}
\label{prop1}
Suppose that for every $x,y$ the loss function $\ell(\mathcal W,x,y)$ is $L$-Lipschitz with respect to $\mathcal W$. Then, the noisy loss function $\widetilde{\ell}(\mathcal{W},x,y)$ is $\frac L \kappa$-smooth with respect to $\mathcal{W}$.
\end{proposition}






\subsection{Notations}
Throughout the paper, the norm $\Vert 
\cdot\Vert$ refers to the $\ell_2$-norm for both vectors and matrices. We assume that the $\ell_2$-matrix norm of each weight matrix $W_i$ is bounded by $B_i$, and the input image $\ell_2$-norm or $\Vert x\Vert$ is upper bounded by $C$. We let $x_{\text{min}}$ and $x_{\text{max}}$ denote the minimum and maximum possible values of all pixels in $x$, respectively. Thus, $x_{\text{max}} - x_{\text{min}}$ denotes the pixel value range.
To simplify our formula, we use the notation $\Delta_{t, i}$ as the sum of the product of all $(t-i)$-subsets of $B_1,\dots, B_t$. In particular, we define
$\Delta_{t,0}=\prod_{i=1}^tB_i,\: \Delta_{t,1}=\sum_{i=1}^t\prod_{j\ne i,j\le t}B_j$.

\subsection{Saliency Maps}
To interpret the classification output of a trained neural network for input $x$, we use the prediction logits $f_{\mathcal W}(x)$ and are in particular interested in the value of the $y$-th component, where $f_{\mathcal W}(x)$ is the neural network function with parameters $\mathcal W$ and the input $x$. For the training set $S$, we use $\mathcal W=A(S)$ to denote the output of a randomized training algorithm $A$. $\operatorname{Sal}$ denotes a general saliency map algorithm, which takes $A(S)$ as input and can output a gradient-based map $\operatorname{Sal}_{A(S)}(x,y)$ at each labeled data point $(x,y)$.  
In this work, we analyze the following standard gradient-based saliency maps:

\textbf{Simple-Grad.} Simple-Grad~\citep{baehrens2010explain,simonyan2013deep} calculates the gradient of the logit concerning each input pixel: 
\begin{align}
    \operatorname{SimpleGrad}_{A(S)}(x,y)=\nabla_x (f_{A(S)}(x))_y.
\end{align}

\textbf{Smooth-Grad.} Smooth-Grad~\citep{smilkov2017smoothgrad} calculates the mean of the Simple-Grad map evaluated at a perturbed input data with a Gaussian noise
\begin{align}
    &\operatorname{SmoothGrad}_{A(S)}(x,y) \nonumber \\
    =\,&\mathbb{E}_{z\sim N(0,\sigma^2I)}\bigl[\nabla_x (f_{A(S)}(x+z))_y\bigr]
\end{align}

\textbf{Integrated-Gradients.} Integrated-Gradients~\citep{sundararajan2017axiomatic} calculates the integral of the scaled gradient map from a reference point $x_0$ to the given data point $x$: \begin{align}
    &\operatorname{IntegratedGrad}_{A(S)}(x,y) \nonumber \\
    =\,&(x-x_0)\odot \int_0^1\nabla_x (f_{A(S)}(x_0+\alpha(x-x_0)))_y\mathrm{d}\alpha.
\end{align}

%% file: 3-Stability-of-Maps.tex
In this work, we study the algorithmic stability of saliency maps. To this end, we first define a loss function for saliency maps \begin{align}
    \ell'(A(S),x,y)=\Vert \operatorname{Sal}_{A(S)}(x,y)-\operatorname{Sal}_{A(D)}(x,y)\Vert,
\end{align} where the reference saliency map is defined as $\operatorname{Sal}_{A(D)}(x,y):=\mathbb{E}_{S,A} [\operatorname{Sal}_{A(S)}(x,y)]$, the expectation of saliency maps across all training datasets $S$ of size $n$ drawn from the underlying data distribution $D$. 
Then we can define the corresponding test loss and training loss by \begin{align}
&L'_D(A(S))=\mathbb{E}_{(x,y)\sim D}[\ell'(A(S),x,y)]\\
&L'_S(A(S))=\mathbb{E}_{(x,y)\sim U(S)}[\ell'(A(S),x,y)] 
\end{align}
Intuitively, a more stable saliency map algorithm to the stochasticity of training data would produce saliency maps with higher similarity for two datasets with only one different sample. Formally, we define the stability error of a saliency map algorithm $\operatorname{Sal}$ in the worst case as follows, where $S$ and $S'$ are two datasets of size $n$ that differ in only one sample and we take the expectation over the randomness of the training algorithm $A$:
\begin{align}
    \epsilon_{\mathrm{stability}}(\operatorname{Sal}):=\sup_{S,S',x,y}\mathbb{E}_{A}\Bigl[&\ell'(A(S),x,y) \nonumber \\
    &-\ell'(A(S'),x,y)\Bigr]
\end{align}
We recall the theorem stating that stability implies generalization in expectation~\citep{bousquet2002stability}. In our setting, we extend the original statement and demonstrate that the generalization error of a saliency map is bounded by the stability error of the saliency map algorithm. This indicates the importance of providing theoretical guarantees for stability error.

\begin{theorem}
\label{thm4.1}
With our defined loss function, we can upper bound the generalization error of the saliency map algorithm by its stability error.
\begin{align*}
    &\mathbb{E}_{S,A}[L'_D(A(S))-L'_S(A(S))]\le \epsilon_{\mathrm{stability}}(\operatorname{Sal})
\end{align*}
\end{theorem}

\begin{proof}
We defer the proof to the Appendix.
\end{proof}

In this section, we then focus on the stability of some commonly used saliency map algorithms, aiming to discover what factors make a saliency map algorithm more stable.

\subsection{Stability of Simple-Grad}

To prove a saliency map algorithm is stable, we consider the $\ell_2$ difference of the predicted saliency maps when the two training sets $S$ and $S'$ only differ at one data point. To upper bound the stability error, we analyze that different loss functions of saliency maps, $\ell'(\cdot)$, are bounded and Lipschitz. 

Similar to \citet{hardt2016train}, we consider using the stochastic gradient descent (SGD) algorithm (can be vanilla or noisy version) and need some assumptions for the training loss function.

\begin{assumption}
\label{assumption}

(a) The loss function $\ell(\cdot,x,y)$ is $L$-Lipschitz for all $x,y$, that is, for any $\mathcal V,\mathcal W$ we have  \begin{align*}
    \Vert\ell(\mathcal V,x,y)-\ell(\mathcal W,x,y)\Vert\le L\Vert \mathcal V-\mathcal W\Vert
\end{align*}

(b) The loss function $\ell(\cdot,x,y)$ is $\beta$-smooth for all $x,y$, that is, for any $\mathcal V,\mathcal W$ we have  
\begin{align*}
    \Vert\nabla_{\mathcal V}\ell(\mathcal V,x,y)-\nabla_{\mathcal W}\ell(\mathcal W,x,y)\Vert\le \beta\Vert \mathcal V-\mathcal W\Vert
\end{align*}
\end{assumption}

\begin{remark}
(1) For noisy SGD introduced in Equation \ref{smooth-SGD}, we can upper bound $L\le \Delta_{k,1}C$ due to lemma \ref{lemma:lipschitz_smooth}, and upper bound $\beta\le \frac{L}{\kappa}$ according to Proposition \ref{prop1}.

(2) For regular SGD, the upper bound for $L$ still holds. However, based on lemma \ref{lemma:lipschitz_smooth}, we need to additionally assume both the activation function $\phi(\cdot)$ and the loss function which takes as input the logit $\ell(\cdot,y)$ for any $y$ are 1-smooth to show $\beta\le (3k+1)\Delta_{k,1}^2C^2$.
\end{remark}


The following theorem upper bounds the stability error of Simple-Grad in terms of parameters of the SGD algorithm ($T$ and $c$), the Lipschitz and smoothness constants of the loss function ($L$ and $\beta$), and the quantity related to the neural network ($\Gamma$). Here we refer to both the vanilla and noisy versions as the SGD algorithm. 


\begin{theorem}
\label{thm4.2}
Suppose Assumption \ref{assumption}(a) and \ref{assumption}(b) hold. If we run SGD for $T$ steps with a decaying step size $\alpha_t\le c/t$ in iteration $t$, by defining $\Gamma=(2\Delta_{k,0}\sum_{i=1}^{k-1}\Delta_{i,1}C)^{1/(\beta c+1)} (2\Delta_{k,0})^{\beta c/(\beta c+1)}$, we can bound the stability error of Simple-Grad by
\begin{align*}
    \epsilon_{\mathrm{stability}}(\operatorname{SimpleGrad})&\le \operatorname{Up}(\operatorname{SimpleGrad}) \\
    &:=\, \frac{1+\beta c}{n-1}(2cL)^{\frac{1}{\beta c+1}}T^{\frac{\beta c}{\beta c+1}}\Gamma
\end{align*}
\end{theorem}

\begin{proof}
We defer the proof to the Appendix.
\end{proof}






\subsection{Stability of Smooth-Grad}
For Smooth-Grad with a simple adjustment where we perform normalization to ensure the input to the neural network has an $\ell_2$ norm not exceeding $C$, 
we can prove the following upper bound. 
\begin{theorem}
\label{thm4.4}
Suppose Assumption \ref{assumption}(a) and \ref{assumption}(b) hold. If we run SGD for $T$ steps with a decaying step size $\alpha_t\le c/t$ in iteration $t$, we can bound the stability error of Smooth-Grad by
\begin{align*}
\epsilon_{\mathrm{stability}}(\operatorname{SmoothGrad})&\le \operatorname{Up}(\operatorname{SimpleGrad}) \\
&\times \left(\frac{\Delta_{k,1}}{2\sigma\Delta_{k,0}\sum_{i=1}^{k-1}\Delta_{i,1}}\right)^{\frac{1}{\beta c+1}}
\end{align*}
\end{theorem}

\begin{proof}
We defer the proof to the Appendix.
\end{proof}

Comparing Theorem \ref{thm4.2} and Theorem \ref{thm4.4}, we observe that Smooth-Grad exhibits a lower stability error by multiplying a dimension-free constant smaller than 1. Theorem \ref{thm4.4} also provides a non-asymptotic dimension-free guarantee that as the $\sigma$ value becomes larger, the stability error vanishes to $0$. This suggests that the smoothing filter enhances the algorithmic stability of high-dimensional saliency maps.



%% file: 4-FideltyMaps.tex
We then illustrate that this stability improvement comes at the expense of fidelity. As $\sigma$ increases, the Gaussian-smoothed saliency map diverges further from the ground-truth map.

First, we formally define the fidelity error $\epsilon_{\mathrm{fidelity}}(\operatorname{Sal},\sigma)$ of the $\sigma$-smoothed saliency map algorithm $\operatorname{Sal}$ as the largest possible expected $\ell_2$-distance between the smoothed and unsmoothed saliency maps for any training set $S$, labeled data sample $x,y$, where the expectation is over the randomness of the algorithm. Specifically, fidelity error is defined as:
\begin{align}
    \sup_{S,x,y}\mathbb E_{A}\Bigl[\Big\Vert &\mathbb{E}_{z\sim N(0,\sigma^2I)}\left[\operatorname{Sal}_{A(S)}(x+z,y)\right] \nonumber \\
    &-\operatorname{Sal}_{A(S)}(x,y)\Big\Vert\Bigr]
\end{align}

We present the following proposition to upper bound the fidelity error of Smooth-Grad and the smoothed version of Integrated-Grad, indicating the fidelity error grows with the Gaussian smooth factor $\sigma$. 
\begin{proposition}
\label{thm5.1}
    The fidelity error of both Smooth-Grad and the smoothed version of Integrated-Grad increase accordingly with the Guassian smoothing factor $\sigma$. In particular, by defining $\Lambda=\Delta_{k,0}\sum_{i=1}^{k-1}\Delta_{i,0}$, we have \begin{align*}
        \epsilon_{\mathrm{fidelity}}(\operatorname{SimpleGrad},\sigma)&\le \Lambda\sigma\sqrt{m}\\
        \epsilon_{\mathrm{fidelity}}(\operatorname{IntegratedGrad},\sigma)&\le \Lambda(3C+1)\sigma\sqrt{m}
    \end{align*}
\end{proposition}

\begin{proof}
We defer the proof to the Appendix.
\end{proof}

%% file: 5-NumericalResults.tex
\begin{figure*}[!ht]
	\centering
    \includegraphics[width=0.95\linewidth]{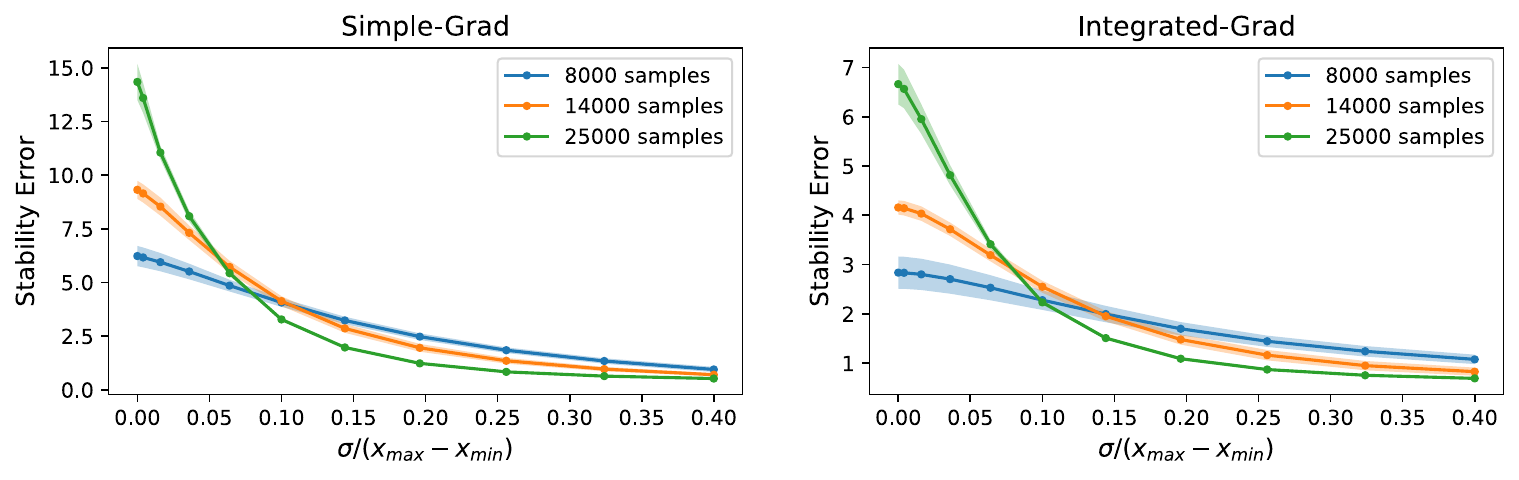}
    \vspace{-0.2cm}
	\caption{Relationship between stability error of saliency maps and sigma, evaluated on the CIFAR10 test set. The shaded region indicates one standard deviation from the mean value. Here noise level $x_{\text{max}}-x_{\text{min}}$ represents the largest value range for all pixels, following the practice of ~\citet{smilkov2017smoothgrad}.}
	\label{fig:cifar10_plot_sigma}
\end{figure*}

\begin{figure*}[t]
	\centering
	\includegraphics[width=0.95\linewidth]{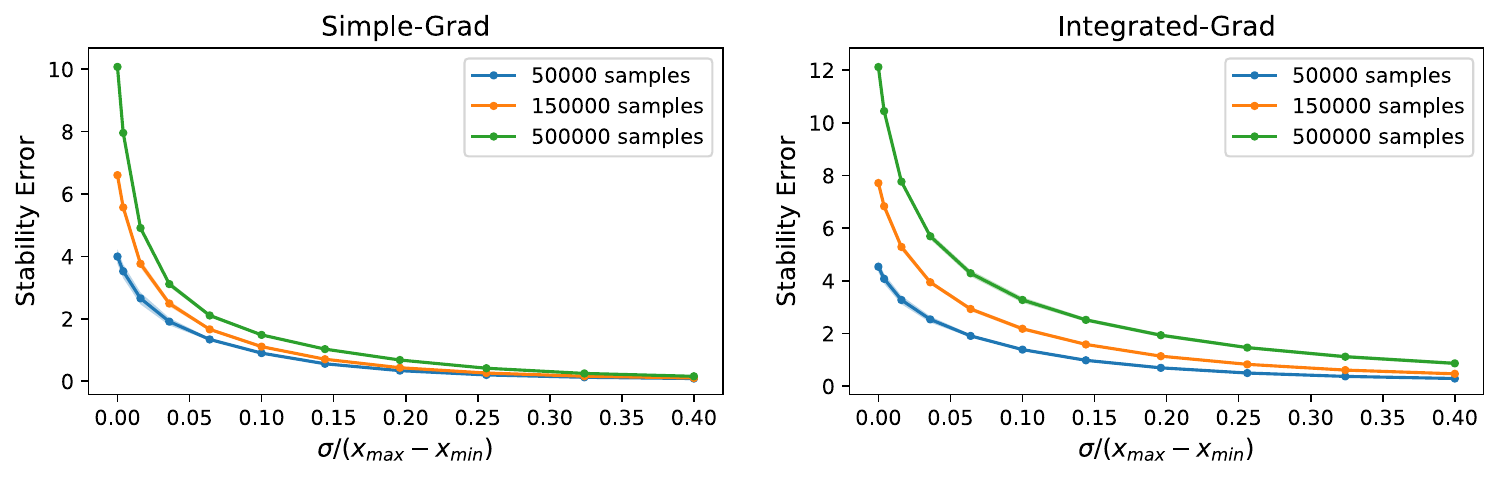}
    \vspace{-0.2cm}
	\caption{Relationship between stability error of saliency maps and sigma, evaluated on the ImageNet test set. The shaded region indicates one standard deviation from the mean value. }
	\label{fig:imagenet_plot_sigma}
\end{figure*}

\begin{figure*}[t]
	\centering
	\includegraphics[width=0.95\linewidth]{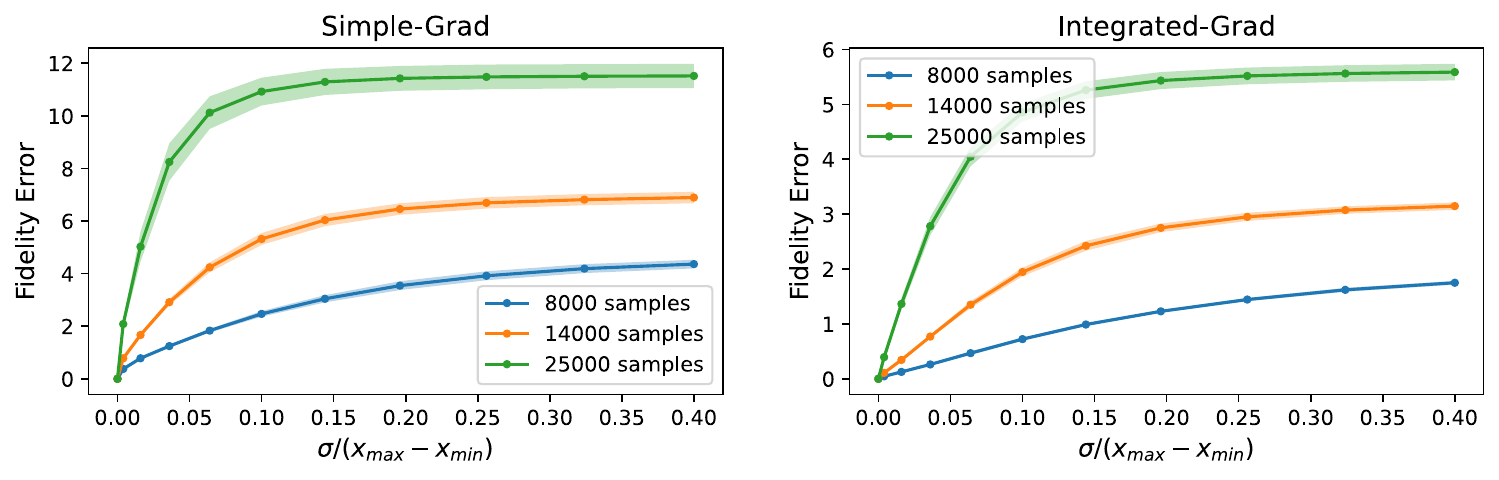}
    \vspace{-0.2cm}
	\caption{The relation between the saliency map fidelity error and sigma. This experiment is conducted on the test set of CIFAR10. The shaded area represents one standard deviation.}
	\label{fig:fidelity_cifar10_plot_sigma}
\end{figure*}

\begin{figure*}[t]
	\centering
	\includegraphics[width=0.95\linewidth]{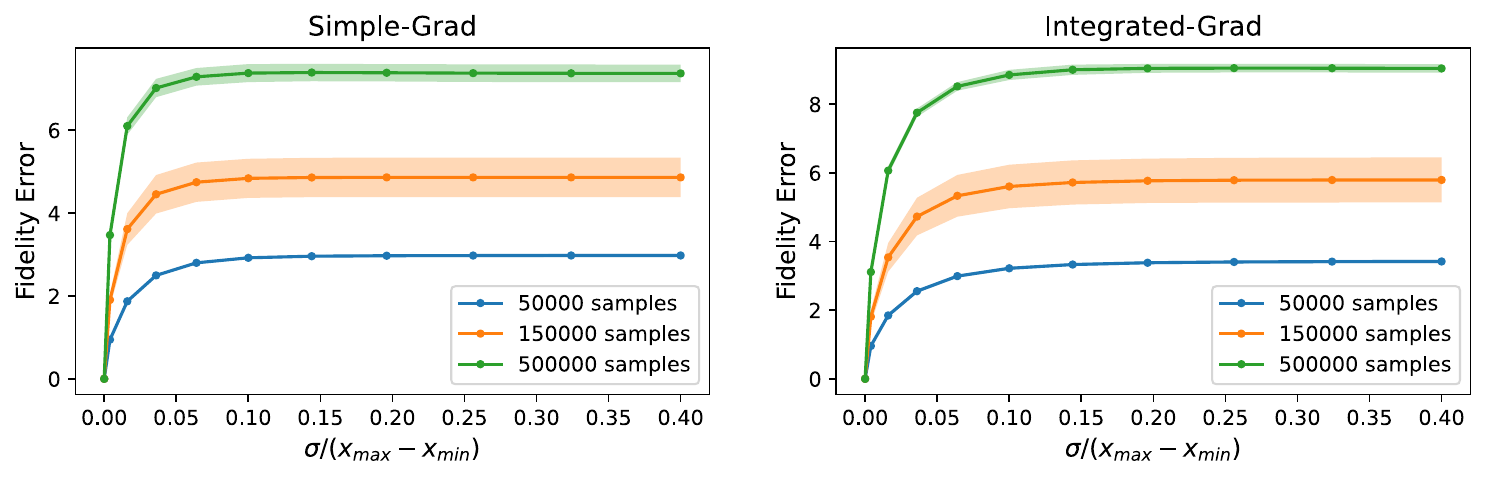}
    \vspace{-0.2cm}
	\caption{The relation between the saliency map fidelity error and sigma. This experiment is conducted on the test set of ImageNet. The shaded area represents one standard deviation.}
	\label{fig:fidelity_imagenet_plot_sigma}
\end{figure*}

The goal of our numerical experiments is to validate the stability-fidelity tradeoff incurred by choosing different smoothing factors $\sigma$.
We evaluated different neural network architectures and different training set sizes on CIFAR10~\citep{Krizhevsky09} and ImageNet~\citep{5206848} datasets. We trained ResNet34~\citep{he2015deep} for CIFAR10 and ResNet50 for ImageNet, due to different dataset scales. More detailed experiment configurations are in Appendix~\ref{appendix:detail}.

\subsection{Saliency Map Stability and Fidelity Numerical Results}
To evaluate the algorithmic stability of saliency maps, we need to perturb one training sample multiple times and train a neural network from scratch for each perturbed training set $S'$. However, this approach is computationally prohibitive. As an alternative, we randomly divide the original training set $S$ into two disjoint subsets $S_1$ and $S_2$ so that the two training subsets can be regarded as independently sampled from the test distribution, and train two neural networks $A(S_1)$ and $A(S_2)$ on them. We then use the average $\ell_2$-difference between two predicted saliency maps as a proxy for stability error, formulated as \begin{align}
    \mathbb{E}_{(x,y)\sim D}\Vert \operatorname{Sal}_{A(S_1)}(x,y)-\operatorname{Sal}_{A(S_2)}(x,y)\Vert
\end{align}

We conducted experiments on CIFAR10 in Figure \ref{fig:cifar10_plot_sigma} and ImageNet in Figure \ref{fig:imagenet_plot_sigma} for different saliency map algorithms as the smoothing factor $\sigma$ changes. Our results confirm that increasing $\sigma$ consistently decreases the stability error of the saliency maps.

To evaluate the fidelity error of a $\sigma$-smoothed saliency map, we similarly calculate
\begin{align}
    \mathbb{E}_{(x,y)\sim D}\mathbb{E}_{i=1}^2\bigl\Vert &\mathbb{E}_{z\sim N(0,\sigma^2I)}[\operatorname{Sal}_{A(S_i)}(x+z,y)] \nonumber \\
    &-\operatorname{Sal}_{A(S_i)}(x,y)\bigr\Vert
\end{align} as a proxy, where $\operatorname{Sal}_{A(S)}(x,y)$ represents a non-smoothed saliency map.

We also performed experiments on CIFAR10 in Figure \ref{fig:fidelity_cifar10_plot_sigma} and ImageNet in Figure \ref{fig:fidelity_imagenet_plot_sigma} to evaluate the fidelity of the saliency maps for different saliency map algorithms as the smoothing factor $\sigma$ varies. The results show that the fidelity error increases consistently as $\sigma$ increases, supporting our theoretical findings.


\input{tables/results}

\begin{figure*}[t]
\centering
\includegraphics[width=\linewidth]{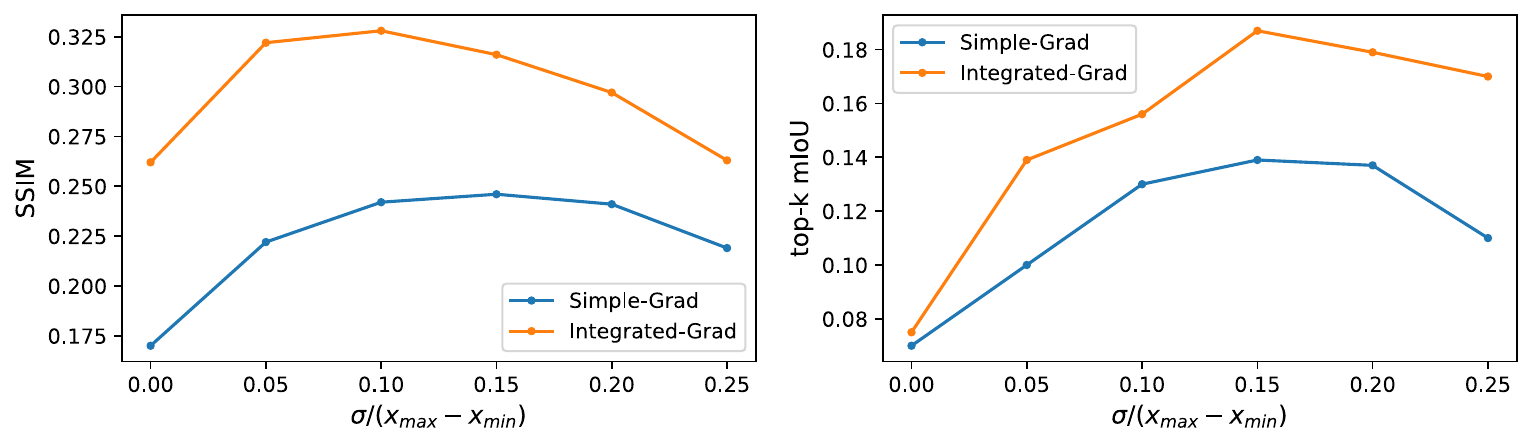}
\vspace{-0.7cm}
\caption{The effect of Gaussian smoothing with different $\sigma$ choices on SSIM and top-k mIoU for Simple-Grad and Integrated-Grad, where the two neural networks are trained on two splits of the ImageNet training set.}
\label{fig:ssim_sigma}
\end{figure*}

\subsection{Visualizations of Saliency Map Stability and Fidelity}
We visualize the differences between two saliency maps when training two neural networks on different training sets. From the results of two neural networks trained on disjoint training subsets of $500,000$ samples from ImageNet in Figure \ref{fig:visualization_Simple}, we observe that Gaussian smoothing makes the saliency maps more similar and thus more stable to the randomness of training samples.



\begin{figure*}[t]
  \centering
  \includegraphics[width=0.8\linewidth]{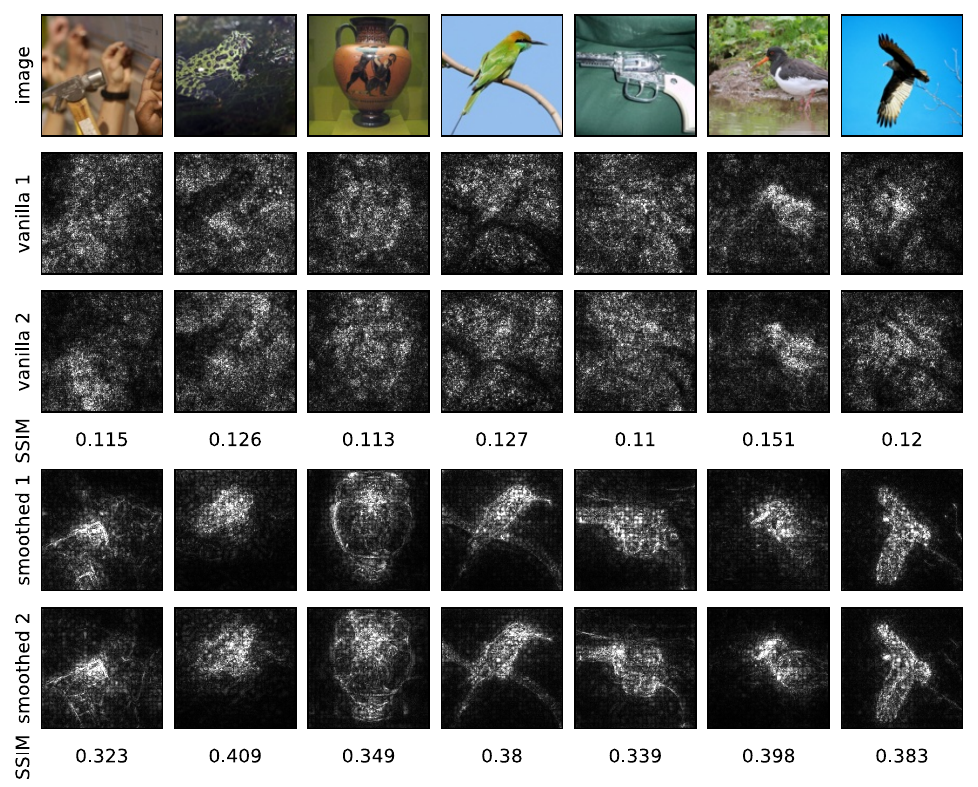}
  \vspace{-0.2cm}
  \caption{Visualization of saliency map differences for two ResNet50 models trained on disjoint subsets of the ImageNet training set, each containing 500,000 random samples. The first row shows the input image. The second and third rows display the Simple-Grad results. The fourth and fifth rows show the Smooth-Grad results with $\sigma/(x_{\mathrm{max}}-x_{\mathrm{min}}) = 0.15$.}
  \label{fig:visualization_Simple}
\end{figure*}

\begin{figure*}[!t]
    \centering
    \begin{minipage}{0.8\linewidth}
        \centering
        \includegraphics[width=1\linewidth]{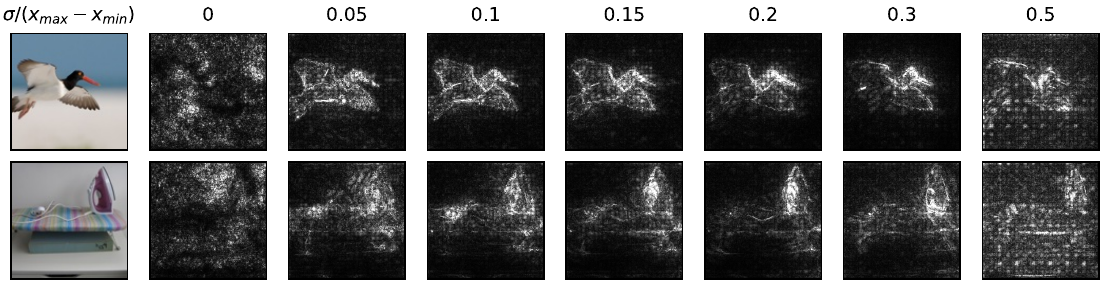}
    \end{minipage}
    \vspace{-0.1cm}
    \caption{Visualization of Simple-Grad becoming less similar to the original unsmoothed map with the Gaussian smoothing as the factor $\sigma/(x_{\mathrm{max}}-x_{\mathrm{min}})$ increases. }
    \vspace{-0.1cm}
    \label{fig:vis_fidelity}
\end{figure*}

Additionally, we visualize the trend that a larger $\sigma$ results in a smoothed saliency map that is less similar to the vanilla version, as demonstrated in Figure \ref{fig:vis_fidelity} for Simple-Grad and in Figure \ref{fig:vis_fidelity_integrated} for Integrated-Grad. 

The visualization results, along with our theoretical findings, indicate that Gaussian smoothing enhances the stability of saliency maps at the expense of fidelity. Therefore, the choice of $\sigma$ in practical applications should be carefully determined by considering this trade-off, since applications of saliency maps could require different levels of the two factors.

\subsection{Empirical Results on Generalized Settings}
Beyond the randomness of training settings, we further examine other sources that could cause instability in interpretation: different training recipes and model architectures. Extensive experiments across various settings, as shown in Table \ref{tab:results}, empirically indicate that the stabilizing effect of Gaussian smoothing also holds in different scenarios. We use the structural similarity index measure  SSIM~\citep{wang2004image} to evaluate the similarity of two saliency maps and a larger SSIM value indicates that the two maps are more perceptually similar. We additionally use top-k mIoU to evaluate the similarity by calculating mIoU between the most salient $k=500$ pixels of two maps.

In the top four rows, we train two ResNet50 models on two disjoint ImageNet training subsets of 500,000 samples each. For all other rows, we use pre-trained checkpoints provided by TorchVision~\citep{torchvision2016}. Since TorchVision provides two ResNet50 checkpoints trained with different schemes, 
we use this pair for the middle four rows of our table. For the ResNet50 models in the bottom rows, we use the checkpoint \texttt{version 1}.

First, when networks are trained on different training sets, Gaussian smoothing makes the resulting saliency maps more similar for both Simple-Grad and Integrated-Grad. For advanced saliency map algorithms such as Grad-CAM~\citep{selvaraju2017grad} and Grad-CAM++~\citep{chattopadhay2018generalized}, the vanilla maps have high SSIM values due to the concentrated and block-like patterns, yet their top-k mIoU are still modest. Applying Gaussian smoothing occasionally causes significant visual changes in one saliency map but not the other, resulting in a slightly lower SSIM value for Grad-CAM. 
Second, the instability caused by different training recipes can be largely mitigated by Gaussian smoothing, with better SSIM and top-k mIoU scores. See the corresponding visualization for Simple-Grad in Figure~\ref{fig:vis_pretrained} in Appendix \ref{appendix:vis}.

Third, the effect of smoothing holds for most cross-architecture experiments. Besides ResNet, we also evaluated more recent architectures such as Swin Transformer~\citep{liu2021Swin} and ConvNeXt~\citep{liu2022convnet}. See Figures \ref{fig:vis_arch}, \ref{fig:vis_arch2}, and \ref{fig:vis_arch3} for corresponding visualizations in Appendix \ref{appendix:vis}. Finally, saliency map algorithms like Grad-CAM and Grad-CAM++ also enjoy improved stability through Gaussian smoothing in some scenarios according to Table~\ref{tab:results}. See Figures \ref{fig:vis_Integrated}, \ref{fig:vis_gradcam}, and \ref{fig:vis_gradcam++} for visualizations of Integrated-Grad, Grad-CAM, and Grad-CAM++ in Appendix \ref{appendix:vis}. Figure~\ref{fig:ssim_sigma} shows that the highest stabilizing effect of Gaussian smoothing can be achieved with a proper choice of $\sigma$. 

%% file: tables/results.tex
%
\newcommand{\tablestyle}[2]{\setlength{\tabcolsep}{#1}\renewcommand{\arraystretch}{#2}\centering\footnotesize}

\begin{table*}
\caption{Average SSIM and top-k mIoU scores for saliency maps from two neural networks evaluated on a random subset of the ImageNet test set. Smoothed scores are computed after applying Gaussian smoothing to both saliency maps. We consider various settings including different training sets (DT), different recipes for training (DR), different architectures of neural networks (DA). We use $\sigma/(x_{\mathrm{max}}-x_{\mathrm{min}})=0.15$ and draw $100$ samples to calculate smoothed saliency maps in all settings. \textbf{Bold} represents the better score.}
\label{tab:results}
\centering
\resizebox{\linewidth}{!}{%
\begin{tabular}{ccccccc}
\toprule
\textbf{Setting} & \textbf{Algorithm} & \textbf{Architecture} & \textbf{SSIM} & \textbf{Smoothed SSIM} & \textbf{top-k mIoU} & \textbf{Smoothed top-k mIoU} \\
\midrule
DT & Simple-Grad & ResNet50 & 0.163 & \textbf{0.227} & 0.061 & \textbf{0.122} \\
DT & Integrated-Grad & ResNet50 & 0.270 & \textbf{0.326} & 0.085 & \textbf{0.163} \\
DT & Grad-CAM & ResNet50 & \textbf{0.872} & 0.846 & 0.110 & \textbf{0.158} \\
DT & Grad-CAM++ & ResNet50 & 0.874 & \textbf{0.886} & 0.127 & \textbf{0.221} \\
\midrule
DR & Simple-Grad & ResNet50 & 0.150 & \textbf{0.179} & 0.050 & \textbf{0.092}\\
DR & Integrated-Grad & ResNet50 & 0.259 & \textbf{0.294} & 0.072 & \textbf{0.141}\\
DR & Grad-CAM & ResNet50 & 0.729 & \textbf{0.754} & 0.070 & \textbf{0.165} \\
DR & Grad-CAM++ & ResNet50 & 0.725 & \textbf{0.743} & 0.047 & \textbf{0.174} \\
\midrule
DA & Simple-Grad & ResNet34 \& ResNet50 & 0.141 & \textbf{0.180} & 0.053 & \textbf{0.112} \\
DA & Integrated-Grad & ResNet34 \& ResNet50 & 0.249 & \textbf{0.279} & 0.064 & \textbf{0.150} \\
DA & Simple-Grad & ResNet50 \& Swin-T & 0.095 & \textbf{0.140} & 0.025 & \textbf{0.068} \\
DA & Integrated-Grad & ResNet50 \& Swin-T & 0.172 & \textbf{0.202} & 0.031 & \textbf{0.088} \\
DA & Simple-Grad & ResNet50 \& ConvNeXt-T & 0.102 & \textbf{0.167} & 0.035 & \textbf{0.079} \\
DA & Integrated-Grad & ResNet50 \& ConvNeXt-T & 0.208 & \textbf{0.258} & 0.041 & \textbf{0.109} \\
DA & Simple-Grad & Swin-T \& ConvNeXt-T & \textbf{0.201} & \textbf{0.201} & 0.047 & \textbf{0.100} \\
DA & Integrated-Grad & Swin-T \& ConvNeXt-T & \textbf{0.268} & 0.242 & 0.053 & \textbf{0.121} \\
\bottomrule
\end{tabular}%
}
\end{table*}

%% file: 6-Conclusion.tex
In this work, we analyzed the stability fidelity trade-off in the widely-used Smooth-Grad interpretation map. Our results indicate the impact of Gaussian smoothing on the stability of saliency maps to the stochasticity of the training setting. On the other hand, we show Gaussian smoothing could lead to a higher discrepancy with the original gradient map. Therefore, Gaussian smoothing needs to be utilized properly, depending on the target application's aimed stability and fidelity scores. Performing an analytical study of the effects of Gaussian smoothing on the sanity checks in \citet{adebayo2018sanity} is an interesting future direction to our work. While \citet{adebayo2018sanity} numerically shows that Smooth-Grad passes the sanity checks, an analysis of the $\sigma$ parameter in Gaussian smoothing and Smooth-Grad's performance in the sanity checks will be a relevant topic for future study. 


%% file: 7-Appendix.tex
\section{More Results}
\label{appendix:vis}
Due to the space limit of the main text, we put more results here. 
Figure~\ref{fig:vis_fidelity_integrated} visualizes the trend of less faithfulness to the original saliency map with increasing $\sigma$.
Figure \ref{fig:vis_pretrained} displays two ResNet50 checkpoints trained with different recipes. 
Figure \ref{fig:vis_arch}, \ref{fig:vis_arch2}, \ref{fig:vis_arch3} show our findings extend to the case where two neural networks use different architectures (and naturally the training recipes also differ).
Figure \ref{fig:vis_Integrated}, \ref{fig:vis_gradcam}, \ref{fig:vis_gradcam++} shows the stabilizing effect of Gaussian smoothing for Integrated-Grad, Grad-CAM, and Grad-CAM++, respectively. 
These visualizations, together with Table~\ref{tab:results} in the main text, provide empirical evidence that Gaussian smoothing can generally stabilize saliency maps across various settings.

\begin{figure*}[!ht]
    \centering
    \begin{minipage}{0.8\linewidth}
        \centering
        \includegraphics[width=1\linewidth]{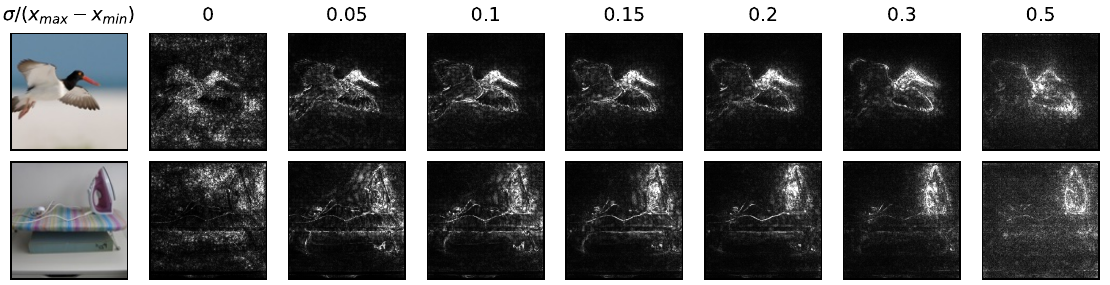}
    \end{minipage}
    \caption{Visualization of Integrated-Grad becoming less similar to the original unsmoothed map with the Gaussian smoothing as the factor $\sigma$ increases. }
    \label{fig:vis_fidelity_integrated}
\end{figure*}

\begin{figure*}[!ht]
  \vspace{-0.2cm}
  \centering
  \includegraphics[width=0.7\textwidth]{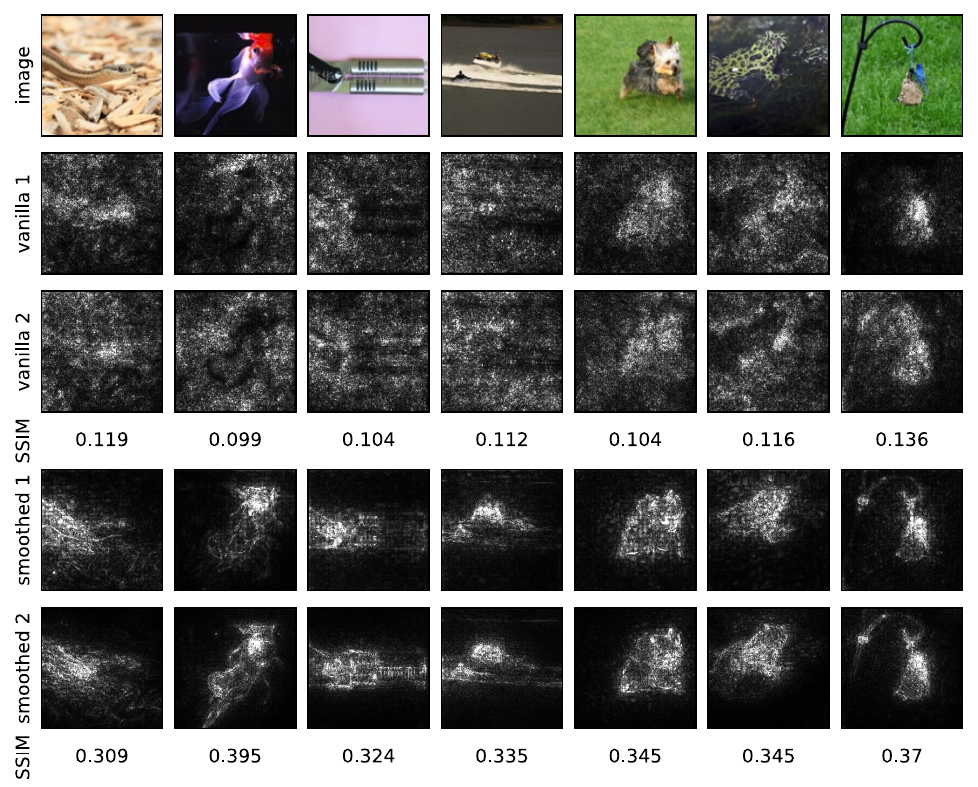}
  \vspace{-0.4cm}
  \caption{Visualization of saliency maps of two pre-trained ResNet50 checkpoints provided by TorchVision~\citep{torchvision2016}. The first row is the input image. The second and third rows are Simple-Grad results. The fourth and fifth rows are Smooth-Grad with $\sigma/(x_{\mathrm{max}}-x_{\mathrm{min}})=0.15$.}
  \label{fig:vis_pretrained}
  \vspace{-0.45cm}
\end{figure*}
\begin{figure*}[!ht]
  \centering
  \includegraphics[width=0.7\textwidth]{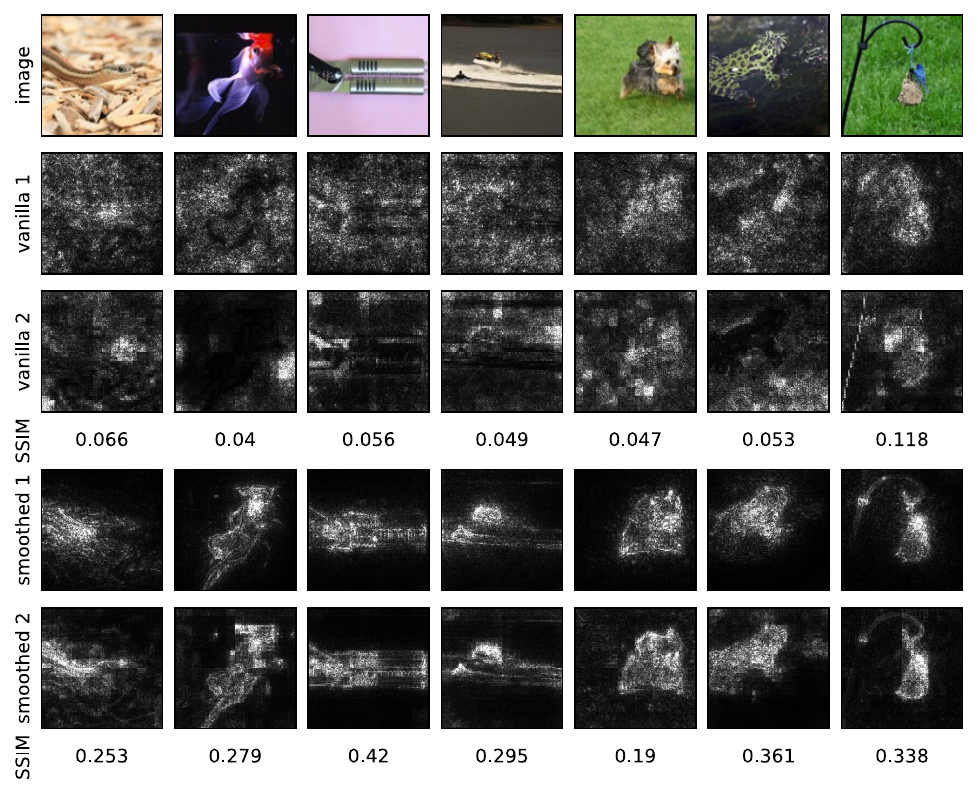}
  \vspace{-0.4cm}
  \caption{Visualization of saliency maps of pre-trained ResNet50 and Swin-T checkpoints provided by TorchVision~\citep{torchvision2016}. The first row is the input image. The second and third rows are Simple-Grad results. The fourth and fifth rows are Smooth-Grad with $\sigma/(x_{\mathrm{max}}-x_{\mathrm{min}})=0.15$.}
  \label{fig:vis_arch}
  \vspace{-0.5cm}
\end{figure*}
\begin{figure*}[!ht]
  \vspace{-0.2cm}
  \centering
  \includegraphics[width=0.7\textwidth]{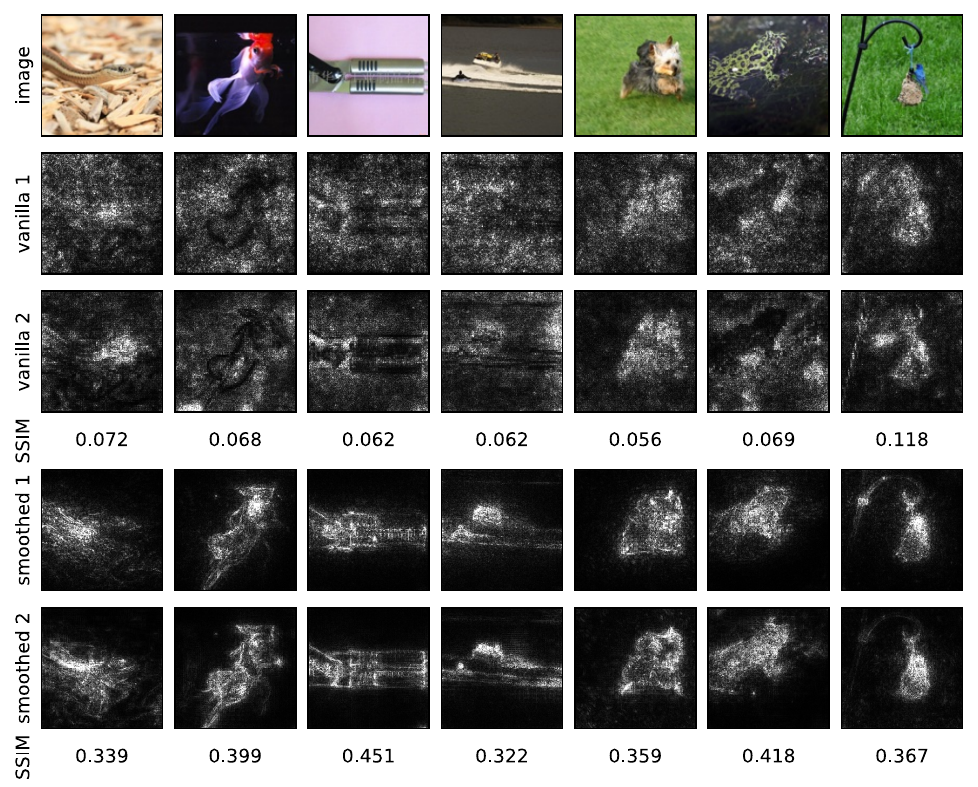}
  \vspace{-0.4cm}
  \caption{Visualization of saliency maps differences of pre-trained ResNet50 and ConvNeXt-tiny checkpoints provided by TorchVision~\citep{torchvision2016}. The first row is the input image. The second and third rows are Simple-Grad results. The fourth and fifth rows are Smooth-Grad with $\sigma/(x_{\mathrm{max}}-x_{\mathrm{min}})=0.15$.}
  \label{fig:vis_arch2}
  \vspace{-0.5cm}
\end{figure*}

\begin{figure*}[!ht]
  \vspace{-0.2cm}
  \centering
  \includegraphics[width=0.7\textwidth]{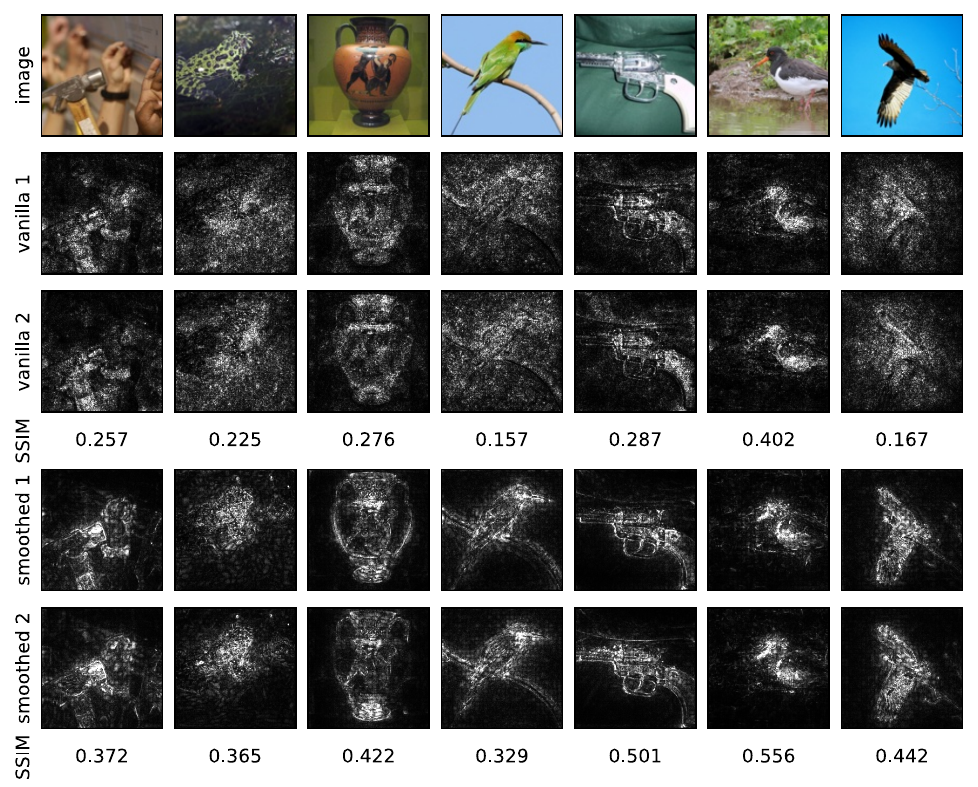}
  \vspace{-0.4cm}
  \caption{Visualization of Integrated-Grad and smoothed Integrated-Grad differences of two ResNet50's trained using two disjoint halves of the training set, each with $500000$ random samples. The first row is the input image. The second and third rows are Integrated-Grad results. The fourth and fifth rows are smoothed Integrated-Grad with $\sigma/(x_{\mathrm{max}}-x_{\mathrm{min}})=0.15$.}
  \label{fig:vis_Integrated}
\vspace{-0.1cm}
\end{figure*}
\begin{figure*}[!ht]
  \centering
  \includegraphics[width=0.7\textwidth]{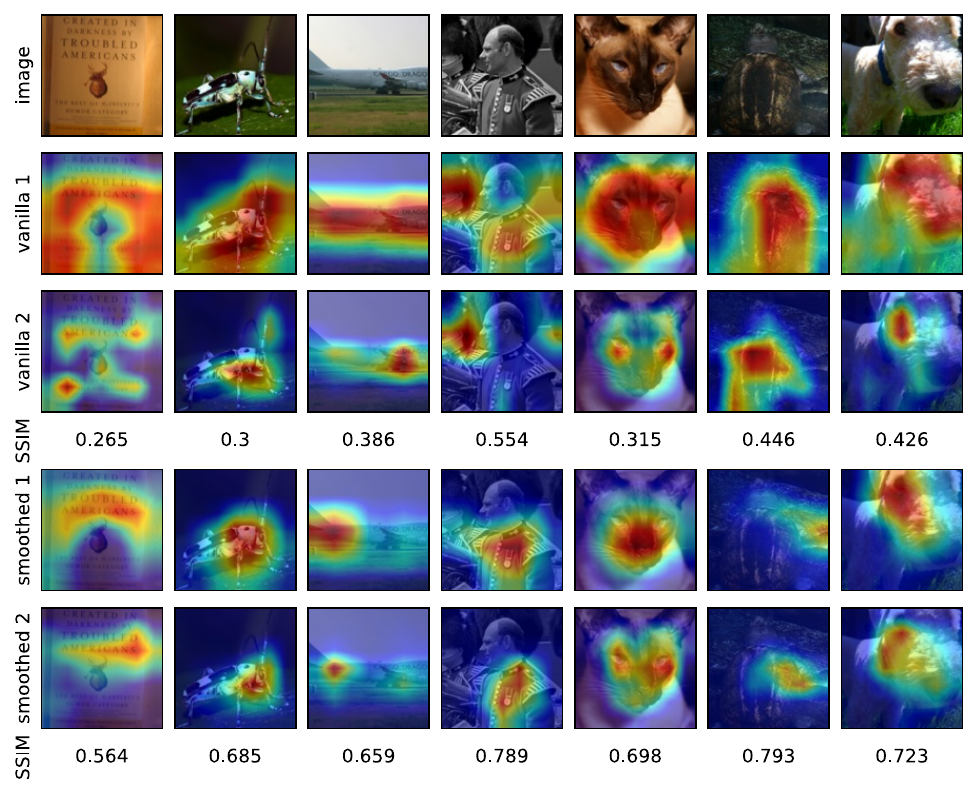}
  \vspace{-0.4cm}
  \caption{Visualization of Grad-CAM and smoothed version of Grad-CAM differences on ImageNet, where the two neural networks are different pre-trained ResNet50 checkpoints provided by TorchVision~\citep{torchvision2016}. The first row is the input image. The second and third rows are Simple-Grad results. The fourth and fifth rows are Smooth-Grad with $\sigma/(x_{\mathrm{max}}-x_{\mathrm{min}})=0.15$.}
  \label{fig:vis_gradcam}
\end{figure*}
\begin{figure*}[!ht]
  \vspace{-0.1cm}
  \centering
  \includegraphics[width=0.7\textwidth]{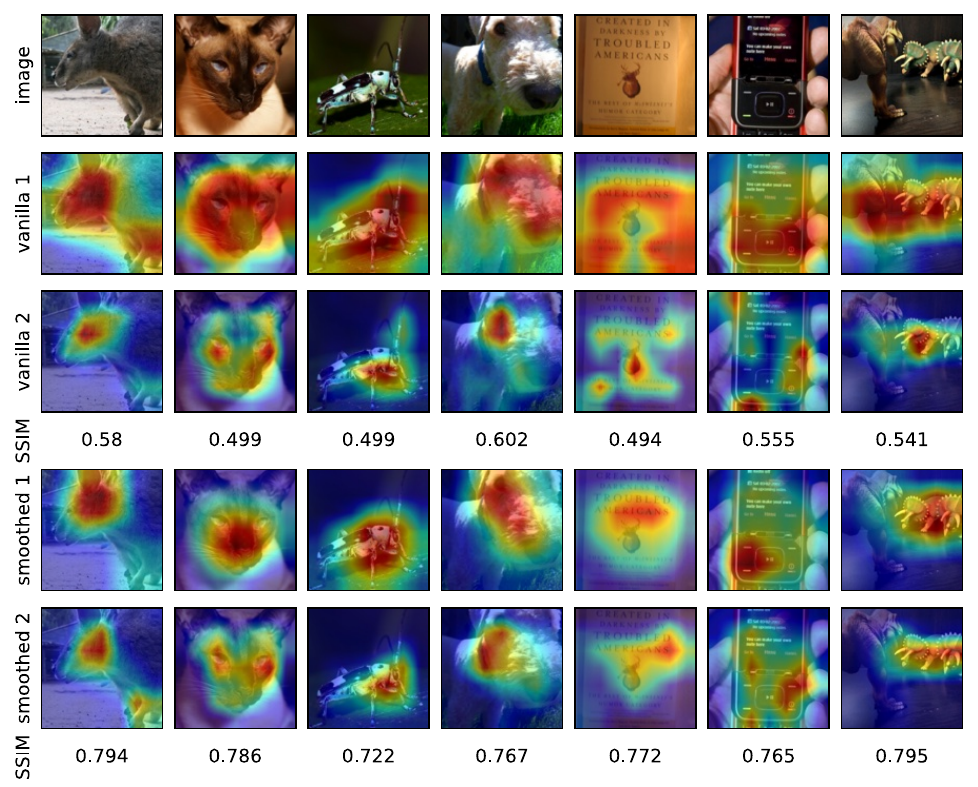}
  \vspace{-0.4cm}
  \caption{Visualization of Grad-CAM++ and smoothed version of Grad-CAM++ differences on ImageNet, where the two neural networks are different pre-trained ResNet50 checkpoints provided by TorchVision~\citep{torchvision2016}. The first row is the input image. The second and third rows are Simple-Grad results. The fourth and fifth rows are Smooth-Grad with $\sigma/(x_{\mathrm{max}}-x_{\mathrm{min}})=0.15$.}
  \label{fig:vis_gradcam++}
\end{figure*}

\clearpage

\begin{figure*}[!ht]
  \vspace{-0.2cm}
  \centering
  \includegraphics[width=0.7\textwidth]{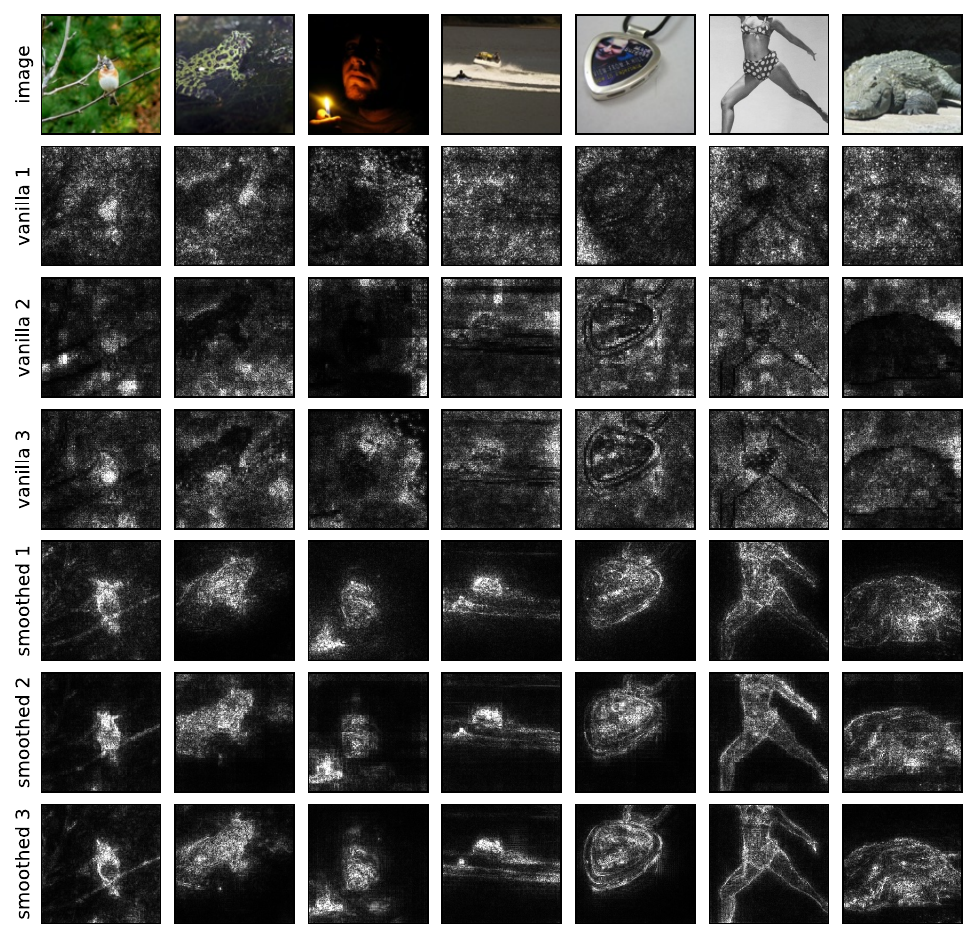}
  \vspace{-0.3cm}
  \caption{Visualization of Simple-Grad and Smooth-Grad differences on ImageNet, where the three neural networks are pre-trained ResNet50, Swin-T, and ConvNeXt-tiny checkpoints provided by TorchVision~\citep{torchvision2016}. The first row is the input image. The second, third, and fourth rows are Simple-Grad results. The fifth, sixth, and seventh rows are Smooth-Grad with $\sigma/(x_{\mathrm{max}}-x_{\mathrm{min}})=0.15$.}
  \label{fig:vis_arch3}
  \vspace{-0.3cm}
\end{figure*}

In addition to Table~\ref{tab:results} in the main text, we consider two more saliency map algorithms Gradient $\odot$ Input~\citep{shrikumar2017learning} and MoreauGrad~\citep{zhang2023moreaugrad} in Table~\ref{tab:addtl_results} and can see the stabilizing effect of Gaussian smoothing still holds.

\begin{table*}
\caption{Average SSIM and top-k mIoU scores for saliency maps from two neural networks evaluated on a random subset of the ImageNet test set. The two ResNet50 networks are trained on two different splits of the training set. Smoothed scores are computed after applying Gaussian smoothing to both saliency maps.}
\label{tab:addtl_results}
\centering
\begin{tabular}{ccccccc}
\toprule
\textbf{Algorithm} & \textbf{SSIM} & \textbf{Smoothed SSIM} & \textbf{top-k mIoU} & \textbf{Smoothed top-k mIoU} \\
\midrule
Gradient $\odot$ Input & 0.264 & \textbf{0.323} & 0.092 & \textbf{0.154} \\
MoreauGrad & 0.621 & \textbf{0.785} & 0.118 & \textbf{0.156} \\
\bottomrule
\end{tabular}%
\end{table*}

\section{Bias-Variance Relation}
To make our study more comprehensive, we also explore the bias-variance relation after using a smoothing kernel. Intuitively, this relation highly correlates to the stability-fidelity relation we focus on. It is natural to define the ground truth saliency map as the vanilla one without smoothing. 

To make the calculation of variance more reliable, we randomly split the training set into ten splits and train ten neural networks independently. Specifically, we define averaged fidelity error as \begin{align*}
    \mathbb{E}_{(x,y)\sim D}\mathbb{E}_{i=1}^{10} \Vert \mathbb{E}_{z\sim N(0,\sigma^2I)}\operatorname{Sal}_{A(S_i)}(x+z,y)-\operatorname{Sal}_{A(S_i)}(x,y)\Vert\end{align*}
and define averaged stability error as
\begin{align*}
    \mathbb{E}_{(x,y)\sim D}\mathbb{E}_{i=1}^{10} \Vert \mathbb{E}_{z\sim N(0,\sigma^2I)}\operatorname{Sal}_{A(S_i)}(x+z,y)-\mathbb{E}_{j=1}^{10} \mathbb{E}_{z\sim N(0,\sigma^2I)}\operatorname{Sal}_{A(S_j)}(x+z,y)\Vert^2
\end{align*} where $\operatorname{Sal}_{A(S)}(x,y)$ is the non-smoothed saliency map.

See the bias-variance relation figures for both Simple-Grad and Integrated-Grad calculated on ImageNet in Figure \ref{fig:bias_variance}. The results summarize our results of stability and fidelity and effectively confirm our main results. We can see that smoothing can reduce variance at the cost of bias. We recommend people choose a suitable smoothing $\sigma$ considering this trade-off in future applications.

\begin{figure*}[ht]
\vspace{-0.2cm}
\centering
\includegraphics[width=\linewidth]{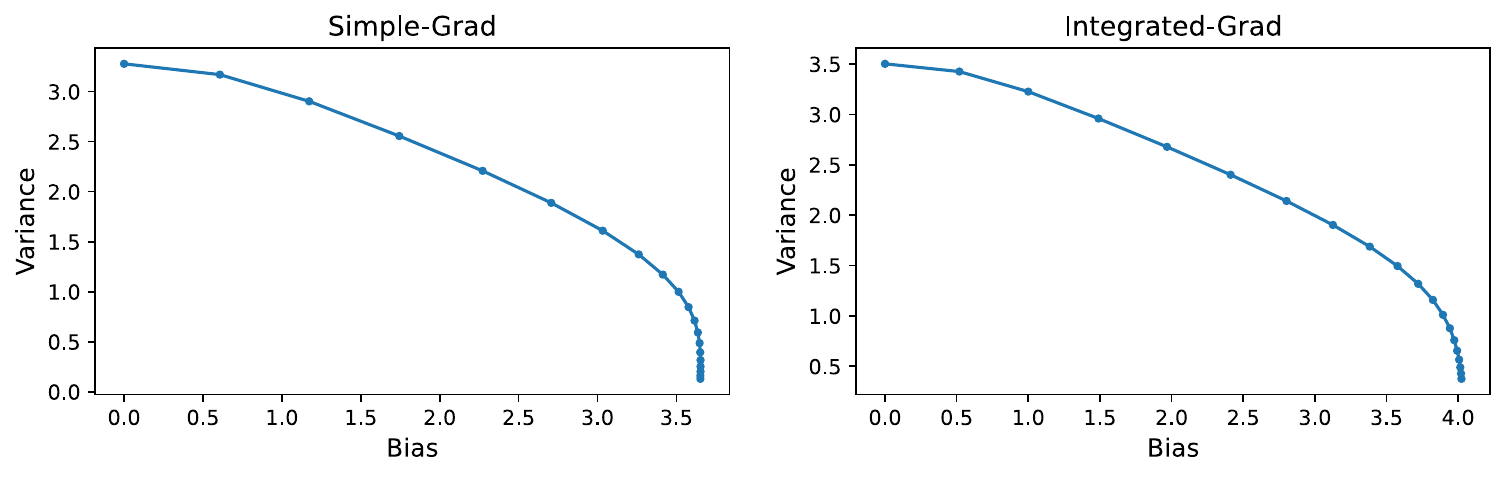}
\vspace{-0.3cm}
\caption{The bias-variance figures for Simple-Grad and Integrated-Grad on ImageNet, by choosing different smoothing factors $\sigma$. When choosing a larger $\sigma$, bias will increase with variance decreasing. In the figure, the square root of $\sigma/(x_{\mathrm{max}}-x_{\mathrm{min}})$ is sampled from 0 to $\sqrt{0.3}$ with equal intervals.}
\label{fig:bias_variance}
\end{figure*}

\section{More Details of Experiments}
\label{appendix:detail}
\subsection{Training Details}
For ImageNet, We train 80 epochs using the SGD optimizer with a momentum of 0.9, the StepLR scheduler with a step size of 30 and a gamma of 0.1 to adjust the learning rate, and weight decay of 0.001. We use \texttt{RandomResizedCrop} to ensure images have shape $224\times 224$, and normalize using the mean and standard deviation computed from ImageNet. All these configurations are default ones in the training script provided by Pytorch.~\footnote{\url{https://github.com/pytorch/examples/blob/main/imagenet/main.py}} For CIFAR10, we train 100 epochs using the SGD optimizer with a momentum of 0.9, cosine learning rate annealing starting from 0.1 down to 0, and weight decay of 0.001. We normalize images using the mean and standard deviation of $0.5$ in all 3 channels. The base models we use are provided by TorchVision~\citep{torchvision2016}. We run experiments on 4 GeForce RTX 3090 GPUs, the training time for a 500000 training subset takes within 12 hours. Except for the training process, most of our experiments run within one hour.

For each specified number of training data points $N$, we conduct the experiments three times, using seeds 10007, 5678, and 12345. In each run, the seed is employed to randomly shuffle the list of all indices. We then select the first $N$ indices to train one neural network, and the subsequent $N$ indices to train another. In this way, we can compute standard deviations for Figure~\ref{fig:cifar10_plot_sigma}, \ref{fig:imagenet_plot_sigma}, \ref{fig:fidelity_cifar10_plot_sigma}, \ref{fig:fidelity_imagenet_plot_sigma}. For the experiment in \ref{fig:bias_variance} which requires 10 splits, we similarly use seed 10007 to shuffle and choose each consecutive $N$ samples as one split. 

\subsection{Implementation of Saliency Map Algorithms}
We use the \texttt{saliency library} provided by PAIR~\footnote{\url{https://github.com/PAIR-code/saliency}} to compute SimpleGrad, IntegratedGrad, and their smoothed versions, as well as visualize those saliency maps. We use the \texttt{pytorch-grad-cam} library~\footnote{\url{https://github.com/jacobgil/pytorch-grad-cam}}~\citep{jacobgilpytorchcam} to compute Grad-CAM and Grad-CAM++ since the former library does not natively support them. We set the desired layer to be \texttt{model.layer4[-1]} for ResNet50 and \texttt{model.norm} for Swin-T, both are close to final layers and our experiments show good visualization results under these choices. According to this library, we visualize a CAM figure as an RGB image by overlaying the cam mask on the image as a heatmap.

When calculating the smoothed version of saliency maps, we draw 100 samples. For IntegratedGrad, we draw 20 intermediate points to compute the integral.

\subsection{Other Details}
Due to computation speed issues, we choose a 512-sample subset of the whole ImageNet test set to compute results for Table~\ref{tab:results} and Figure~\ref{fig:ssim_sigma}. This test subset is created by randomly shuffling the list of all indices with seed 10007 and choosing the first 512 data points, which remains fixed throughout our experiments.

For the calculation of SSIM, we use the one provided by library \texttt{scikit-image}, and set \texttt{gaussian\_weights} to True, sigma to 1.5, \texttt{use\_sample\_covariance} to False, and specify the \texttt{data\_range} argument (to 1 for grey images and 255 for RGB images) according to the official suggestions~\footnote{\url{https://scikit-image.org/docs/stable/api/skimage.metrics.html\#skimage.metrics.structural\_similarity}} to match the implementation of~\citep{wang2004image}. To calculate top-k mIoU, we set $k=500$ pixels for both saliency maps, and calculate the average IoU for these pixels with top-k values according to the definition.

\section{Proofs}
\subsection{Proof of Proposition \ref{prop1}}
We begin by introducing Stein's lemma, as presented in~\citep{lin2019steins}. This lemma is crucial to our overall analysis, as it allows us to express the gradient in expectation in a more manageable form.
\begin{lemma}
    Stein's lemma. For a differentiable function $f$ for which the expectation $\mathbb{E}_{z\sim N(0,\sigma^2 I)}[\nabla_xf(x+z)]$ exists, we have
    $$\mathbb{E}_{z\sim N(0,\sigma^2 I)}[\nabla_xf(x+z)]=\mathbb{E}_{z\sim N(0,\sigma^2 I)}\left[\frac{z}{\sigma^2}f(x+z)\right]$$
\label{lemma:stein}
\end{lemma}

\begin{lemma}
    Suppose function $f:\mathcal X\to \mathcal Y$ is $L$-Lipschitz, then its $\sigma$-Gaussian smoothed version $\Tilde{f}$ defined by $\Tilde{f}(x)=\mathbb{E}_{z\sim N(0,\sigma^2I)}[f(x+z)]$ is $\frac{L}{\sigma}$-smooth.
\end{lemma}
\begin{proof}
    Note that $L$-Lipschitz implies $\Vert \nabla f(x)\Vert\le L$. Define $\Tilde{g}(x)=\nabla_x\Tilde{f}(x)$. Then for any $x_1,x_2$, we have \begin{align*}
        \Vert \Tilde{g}(x_1)-\Tilde{g}(x_2)\Vert&\le \Vert \mathbb{E}_{z\sim N(0,\sigma^2I)}[\nabla_x f(x_1+z)]-\mathbb{E}_{z\sim N(0,\sigma^2I)}[\nabla_x f(x_2+z)]\Vert\\
        &=\left\Vert \mathbb{E}_{z\sim N(0,\sigma^2I)}\left[\frac{z}{\sigma^2}(f(x_1+z)-f(x_2+z))\right]\right\Vert\\
        &=\max_{\Vert u\Vert=1}\mathbb{E}_{z\sim N(0,\sigma^2I)}\left[\frac{u^Tz}{\sigma^2}(f(x_1+z)-f(x_2+z))\right]\\
        &\le \max_{\Vert u\Vert=1}\sqrt{\mathbb{E}_{z\sim N(0,\sigma^2I)}\left(\frac{u^Tz}{\sigma^2}\right)^2\mathbb{E}_{z\sim N(0,\sigma^2I)}(f(x_1+z)-f(x_2+z))^2}\\
        &\le \frac{1}{\sigma}L\Vert x_1-x_2\Vert
    \end{align*}
    where the second inequality is due to Cauchy-Schwarz inequality. Thus $\Tilde{g}(x)$ is $\frac{L}{\sigma}$-Lipschitz, directly implying $\Tilde{f}(x)$ is $\frac{L}{\sigma}$-smooth.
\end{proof}

\subsection{Proof of Theorem \ref{thm4.1}}
Let $S=(z_1,\dots,z_n)$ and $S'=(z_1',\dots,z_n')$ be independent random samples and $S^{(i)}=(z_1,\dots,z_{i-1},z_i',z_{i+1},\dots z_n)$ be identical to $S$ except the the $i$-th position. Then we have 

\begin{align*}
    \mathbb{E}_{S,A}[L'_S(A(S))]&=\mathbb{E}_{S,A}\Bigl[\frac 1n\sum_{i=1}^n\ell'(A(S),x_i,y_i)\Bigr]\\
    &=\mathbb{E}_{S,S',A}\Bigl[\frac 1n\sum_{i=1}^n\ell'(A(S^{(i)}),x_i',y_i')\Bigr]\\
    &=\mathbb{E}_{S,S',A}\Bigl[\frac 1n\sum_{i=1}^n\ell'(A(S),x_i',y_i')\Bigr]-\delta\\
    &=\mathbb{E}_{S,A}[L'_D(A(S))]-\delta
\end{align*}

Then the generalization gap can be written as \begin{align*}
    \delta&=\mathbb{E}_{S,A}[L'_D(A(S))-L'_S(A(S))]\\
    &=\mathbb{E}_{S,S',A}\Bigl[\frac 1n\sum_{i=1}^n\ell'(A(S),x_i',y_i')-\ell'(A(S^{(i)}),x_i',y_i')\Bigr]\\
    &= \mathbb{E}_{S,S',i,A} [\ell'(A(S),x_i',y_i')-\ell'(A(S^{(i)}),x_i',y_i')]
\end{align*}

By taking supremum over any two datasets $S$ and $S'$ differing at only one data point, we can bound 
$$\delta\le \sup_{S,S',x,y}\mathbb{E}_A [\ell'(A(S),x,y)-\ell'(A(S^{(i)}),x,y)]=\epsilon_{\mathrm{stability}}.$$
\hfill$\qedsymbol$

\subsection{Upper Bound of Stability Error for Arbitrary Loss Function}
Throughout the Appendix, we use $f_{\mathcal W}^k(x)$ to denote the neural network function function of the first $k$ layers at input $x$.
\begin{lemma}
\label{lemma:0}
For the class of $k$-layer neural networks $f_{\mathcal W}(x)=W_k\phi(W_{k-1}\phi(\dots\phi(W_1x)))\in \mathbb R^c$ with 1-Lipschitz $\phi$ where $\phi(0) = 0$, we have $$\Vert f_{\mathcal W}(x)\Vert\le \Delta_{k,0}C$$ 
\end{lemma}
\begin{proof}
    By the 1-Lipschitz property of the activation function, we have \begin{align*}
        \Vert f_{\mathcal W}^k(x)\Vert&=\Vert W_k\phi(W_{k-1}\phi(\dots\phi(W_1x)))\Vert\\
        &\le B_k\Vert W_{k-1}\phi(W_{k-2}\phi(\dots\phi(W_1x)))\Vert\\
        &=B_k\Vert f_{\mathcal W}^{k-1}(x)\Vert.
    \end{align*}
    Since we have $f_{\mathcal W}^1(x)\le B_1\Vert x\Vert\le B_1C$, we immediately have the conclusion by induction.
\end{proof}

\begin{lemma}
\label{lemma:1}
For the class of $k$-layer neural networks $f_{\mathcal W}(x)=W_k\phi(W_{k-1}\phi(\dots\phi(W_1x)))\in \mathbb R^c$ with 1-Lipschitz $\phi$ where $\phi(0) = 0$, we have $$\operatorname{lip}_{\mathcal W}(f_{\mathcal W}(x))\le \Delta_{k,1}C$$ 
\end{lemma}
\begin{proof}
    Consider the value difference at $\mathcal W$ and $\mathcal W+\mathcal V$. It is bounded by \begin{align*}
        &\Vert f_{\mathcal W}^k(x)-f_{\mathcal W+\mathcal V}^k(x)\Vert\\
        =\, &\Vert W_k\phi(f_{\mathcal W}^{k-1}(x))-(W_k+V_k)\phi(f_{\mathcal W+\mathcal V}^{k-1}(x))\Vert\\
        \le\, &\Vert W_k\phi(f_{\mathcal W}^{k-1}(x))-(W_k+V_k)\phi(f_{\mathcal W}^{k-1}(x))\Vert\\
        &\: +\Vert(W_k+V_k)\phi(f_{\mathcal W}^{k-1}(x))-(W_k+V_k)\phi(f_{\mathcal W+\mathcal V}^{k-1}(x))\Vert\\
        \le\, & \Vert V_k\Vert\cdot \Vert \phi(f_{\mathcal W}^{k-1}(x))\Vert+\Vert W_k+V_k\Vert\cdot\Vert \phi(f_{\mathcal W}^{k-1}(x))-\phi(f_{\mathcal W+\mathcal V}^{k-1}(x))\Vert\\
        \le\, &\Vert \mathcal V\Vert\cdot \Vert f_{\mathcal W}^{k-1}(x)\Vert+B_k\cdot \operatorname{lip}_{\mathcal W}(f_{\mathcal W}^{k-1}(x))\Vert \mathcal V\Vert\\
        \le\, &\left(\prod_{i=1}^{k-1}B_iC+B_k\operatorname{lip}_{\mathcal W}(f_{\mathcal W}^{k-1}(x))\right)\Vert \mathcal V\Vert,
    \end{align*} where in the last line we apply lemma \ref{lemma:0}.
    It is easy to show $\operatorname{lip}_{\mathcal W}(f_{\mathcal W}^1(x))\le C$, so by induction we have \begin{align*}
        \operatorname{lip}_{\mathcal W}(f_{\mathcal W}^k(x))&\le\prod_{i=1}^{k-1}B_iC+B_k\operatorname{lip}_{\mathcal W}(f_{\mathcal W}^{k-1}(x))\\
        &\le\sum_{i=1}^k\prod_{j\le k,j\ne i}B_jC\\
        &=\Delta_{k,1}C.
    \end{align*} 
\end{proof}

To prepare for our next lemma, which involves taking the gradient with respect to $\mathcal{W}$, we present the following useful result.
\begin{lemma}
    For any matrix $W\in\mathbb R^{n\times m}$ and vector $x\in\mathbb R^c$, we have $$\left\Vert \frac{\partial Wx}{\partial \mathrm{vec}(W)}\right\Vert=\Vert x\Vert$$
\end{lemma}
\begin{proof}
    We can write the gradient as $$\frac{\partial Wx}{\partial \mathrm{vec}(W)}=x^T\otimes I_m.$$ Since the norm of a Kronecker product is the product of the norms of the factors, we immediately have the desired conclusion.
\end{proof}

\begin{lemma}
\label{lemma:2}
For the class of $k$-layer neural networks $f_{\mathcal W}(x)=W_k\phi(W_{k-1}\phi(\dots\phi(W_1x)))\in \mathbb R^c$ with 1-Lipschitz activation $\phi$ where $\phi(0) = 0$, we have $$\Vert\nabla_{\mathcal W}f_{\mathcal W}^k(x)\Vert\le \Delta_{k,1}C$$ 
\end{lemma}
\begin{proof}
The idea is to tackle the whole parameters into two parts inductively. We have \begin{align*}
    \Vert\nabla_{\mathcal W_{1..k}}f_{\mathcal W}^k(x)\Vert&\le\Vert\nabla_{\mathcal W_{1..k-1}}f_{\mathcal W}^k(x)\Vert+\Vert\nabla_{\mathcal W_k}f_{\mathcal W}^k(x)\Vert\\
    &\le\Vert W_k \operatorname{diag}(\phi'(f_{\mathcal W}^{k-1}(x)))\nabla_{\mathcal W_{1..k-1}}f_{\mathcal W}^{k-1}(x)\Vert+\Vert\phi(f_{\mathcal W}^{k-1}(x))\Vert\\
    &\le B_k\Vert\nabla_{\mathcal W_{1..k-1}}f_{\mathcal W}^{k-1}(x)\Vert+\Vert f_{\mathcal W}^{k-1}(x)\Vert\\
    &\le B_k\Vert\nabla_{\mathcal W_{1..k-1}}f_{\mathcal W}^{k-1}(x)\Vert+\prod_{i=1}^{k-1}B_iC,
\end{align*} where in the last inequality we apply lemma \ref{lemma:0}.
Since $\Vert\nabla_{\mathcal W_{1}}f_{\mathcal W}^1(x)\Vert\le \Vert x\Vert\le C$, we can show by induction that \begin{align*}
    \Vert\nabla_{\mathcal W_{1..k}}f_{\mathcal W}^k(x)\Vert&\le\sum_{i=1}^k\prod_{j=i+1}^k B_j\prod_{j=1}^{i-1} B_jC\\
    &=\sum_{i=1}^{k}\prod_{j\le k,j\ne i}B_jC\\
    &=\Delta_{k,1}C.
\end{align*}
\end{proof}

\begin{lemma}
\label{lemma:3}
For the class of $k$-layer neural networks $f_{\mathcal W}(x)=W_k\phi(W_{k-1}\phi(\dots\phi(W_1x)))\in \mathbb R^c$ with 1-Lipschitz, 1-smooth activation $\phi$ where $\phi(0) = 0$, under a mild assumption that $B_i,C\ge 1$, we have $$\operatorname{lip}_{\mathcal W}(\nabla_{\mathcal W}f_{\mathcal W}^k(x))\le 3k\Delta_{k,1}^2C^2$$ 
\end{lemma}
\begin{proof}
    The overall gradient is the concatenation of two parts $\nabla_{\mathcal W_{1..k-1}}f_{\mathcal W}^k(x)$ and $\nabla_{\mathcal W_k}f_{\mathcal W}^k(x)$, and we bound the $\ell_2$-norm of points $\mathcal W$ and $\mathcal W+\mathcal V$ for each of the two parts.

    To bound the first one, we have \begin{align*}
        &\Vert \nabla_{\mathcal W_{1..k-1}}f_{\mathcal W}^k(x)-\nabla_{\mathcal W_{1..k-1}}f_{\mathcal W+\mathcal V}^k(x)\Vert\\
        =\, &\Vert W_k \operatorname{diag}(\phi'(f_{\mathcal W}^{k-1}(x))) \nabla_{\mathcal W_{1..k-1}}f_{\mathcal W}^{k-1}(x)-(W_k+V_k) \operatorname{diag}(\phi'(f_{\mathcal W+\mathcal V}^{k-1}(x))) \nabla_{\mathcal W_{1..k-1}}f_{\mathcal W+\mathcal V}^{k-1}(x)\Vert\\
        \le\, &\Vert W_k \operatorname{diag}(\phi'(f_{\mathcal W}^{k-1}(x))) \nabla_{\mathcal W_{1..k-1}}f_{\mathcal W}^{k-1}(x)-(W_k+V_k) \operatorname{diag}(\phi'(f_{\mathcal W}^{k-1}(x))) \nabla_{\mathcal W_{1..k-1}}f_{\mathcal W}^{k-1}(x)\Vert\\
        &\: +\Vert(W_k+V_k) \operatorname{diag}(\phi'(f_{\mathcal W}^{k-1}(x))) \nabla_{\mathcal W_{1..k-1}}f_{\mathcal W}^{k-1}(x)\\
        &\quad\, -(W_k+V_k) \operatorname{diag}(\phi'(f_{\mathcal W+\mathcal V}^{k-1}(x))) \nabla_{\mathcal W_{1..k-1}}f_{\mathcal W}^{k-1}(x)\Vert\\
        &\: +\Vert(W_k+V_k) \operatorname{diag}(\phi'(f_{\mathcal W+\mathcal V}^{k-1}(x))) \nabla_{\mathcal W_{1..k-1}}f_{\mathcal W}^{k-1}(x)\\
        &\quad\, -(W_k+V_k) \operatorname{diag}(\phi'(f_{\mathcal W+\mathcal V}^{k-1}(x))) \nabla_{\mathcal W_{1..k-1}}f_{\mathcal W+\mathcal V}^{k-1}(x)\Vert\\
        \le\, &\Vert \mathcal V\Vert\cdot \Vert\nabla_{\mathcal W_{1..k-1}}f_{\mathcal W}^{k-1}(x)\Vert+B_k\Vert f_{\mathcal W}^{k-1}(x)-f_{\mathcal W+\mathcal V}^{k-1}(x)\Vert\cdot \Vert\nabla_{\mathcal W_{1..k-1}}f_{\mathcal W}^{k-1}(x)\Vert\\
        &\: +B_k\Vert\nabla_{\mathcal W_{1..k-1}}f_{\mathcal W}^{k-1}(x)-\nabla_{\mathcal W_{1..k-1}}f_{\mathcal W+\mathcal V}^{k-1}(x)\Vert\\
        \le\, &\Vert \mathcal V\Vert\cdot (\Delta_{k-1,1}C+B_k\Delta_{k-1,1}^2C^2+B_k\operatorname{lip}_{\mathcal W}(\nabla_{\mathcal W_{1..k-1}}f_{\mathcal W}^{k-1}(x)))
    \end{align*} where the last line applies lemma \ref{lemma:1} and lemma \ref{lemma:2}.
    For the second one, we have \begin{align*}
        \Vert \nabla_{\mathcal W_k}f_{\mathcal W}^k(x)-\nabla_{\mathcal W_k}f_{\mathcal W+\mathcal V}^k(x)\Vert&=\Vert \phi(f_{\mathcal W}^{k-1}(x))-\phi(f_{\mathcal W+\mathcal V}^{k-1}(x))\Vert\\
        &\le\operatorname{lip}_{\mathcal W}(f_{\mathcal W}^{k-1}(x))\Vert \mathcal V\Vert\\
        &\le\Delta_{k-1,1}C\Vert \mathcal V\Vert&&(\text{applying lemma \ref{lemma:1}})
    \end{align*}
    Combining these, we have \begin{align*}
        \operatorname{lip}_{\mathcal W}(\nabla_{\mathcal W_{1..k}}f_{\mathcal W}^k(x))&\le2\Delta_{k-1,1}C+B_k\Delta_{k-1,1}^2C^2+B_k\operatorname{lip}_{\mathcal W}(\nabla_{\mathcal W_{1..k-1}}f_{\mathcal W}^{k-1}(x))\\
        &\le3B_k\Delta_{k-1,1}^2C^2+B_k\operatorname{lip}_{\mathcal W}(\nabla_{\mathcal W_{1..k-1}}f_{\mathcal W}^{k-1}(x))\\
        &\le3\sum_{i=1}^k\prod_{j=i}^kB_j\Delta_{i-1,1}^2C^2\\
        &\le3\sum_{i=1}^k\left(\prod_{j=i+1}^kB_j\Delta_{i-1,1}\right)\left(B_i\Delta_{i-1,1}\right) C^2\\
        &\le3\sum_{i=1}^k\Delta_{k,1}^2 C^2\\
        &=3k\Delta_{k,1}^2C^2
    \end{align*}

\end{proof}

Building on previous lemmas, we can now show that the training loss function is Lipschitz and smooth.
\begin{lemma}
Consider the class of $k$-layer neural networks $f_{\mathcal W}(x)=W_k\phi(W_{k-1}\phi(\dots\phi(W_1x)))\in \mathbb R^c$ with 1-Lipschitz activation function $\phi$ where $\phi(0) = 0$. Assume that $B_i,C\ge 1$. Also, let the training loss for prediction logits $\ell:(\mathbb R^c,\mathcal Y)\to \mathbb R$ be 1-Lipschitz for all $y\in\mathcal Y$. Then our loss function $\ell(f_{\mathcal W}(x),y)$ is $L$-Lipschitz for all $x,y$. If additionally $\phi(\cdot)$ is 1-smooth and $\ell(\cdot,y)$ is also 1-smooth, then we can show the loss function is $\beta$-smooth for all $x,y$. Here $L\le \Delta_{k,1}C,\beta\le (3k+1)\Delta_{k,1}^2C^2$. 
\label{lemma:lipschitz_smooth}
\end{lemma}
\begin{proof}
    Since $\ell$ is 1-Lipschitz, due to lemma \ref{lemma:1}, we immediately have \begin{align*}
        \operatorname{lip}_{\mathcal W}(\ell(f_{\mathcal W}(x),y))\le \operatorname{lip}_{\mathcal W}(f_{\mathcal W}(x))\le \Delta_{k,1}C.
    \end{align*}
    According to the chain rule and smoothness properties, for all $x,y$, we have \begin{align*}
        \operatorname{lip}_{\mathcal W}(\nabla_{\mathcal W}\ell(f_{\mathcal W}(x),y))&=\operatorname{lip}_{\mathcal W}\left(\nabla\ell(f_{\mathcal W}(x),y)\cdot \nabla_{\mathcal W}f_{\mathcal W}(x)\right)\\
        &\le \sup_{\mathcal W}\Vert \nabla\ell(f_{\mathcal W}(x),y)\Vert \cdot \operatorname{lip}_{\mathcal W}(\nabla_{\mathcal W}f_{\mathcal W}(x))+\operatorname{lip}_{\mathcal W}(\nabla \ell(f_{\mathcal W}(x),y))\cdot \sup_{\mathcal W}\Vert \nabla_{\mathcal W}f_{\mathcal W}(x)\Vert\\
        &\le \operatorname{lip}_{\mathcal W}(\nabla_{\mathcal W}f_{\mathcal W}(x))+\operatorname{lip}_{\mathcal W}(f_{\mathcal W}(x))^2\\
        &\le 3k\Delta_{k,1}^2C^2+\Delta_{k,1}^2C^2\\
        &\le (3k+1)\Delta_{k,1}^2C^2
    \end{align*} Thus, with additional smoothness assumptions, $\ell(f_{\mathcal W}(x),y)$ is $\beta$-smooth with $\beta=(3k+1)\Delta_{k,1}^2C^2$.
\end{proof}

Next, we quote the results from \citep{10.5555/2670022} showing that each SGD step will not make two different parameters much more divergent.

\begin{lemma}
    Assume the function $f$ is $\beta$-smooth. The SGD update rule using $f$ is $w_{t+1}=w_t-\alpha\nabla f(w_t)$. Then, we have \begin{itemize}
        \item [(a)] For all $w_t$ and $w_t'$, we can bound $$\Vert w_{t+1}-w'_{t+1}\Vert\le (1+\alpha\beta)\Vert w_{t}-w'_{t}\Vert$$
        \item [(b)] Assume $f$ is additionally convex. Then for any $\alpha\le 2/\beta$, the following property holds for all $w_t$ and $w_t'$: \begin{align*}
            \Vert w_{t+1}-w'_{t+1}\Vert\le \Vert w_{t}-w'_{t}\Vert
        \end{align*} 
    \end{itemize}
    \label{lemma:expansion}
\end{lemma}

After all these previous results, our next lemma proves the upper bounds for the stability of saliency maps under two different assumptions of the loss function. This proof refers to some analysis in~\citep{hardt2016train}. 

We first state the setting of this lemma. Consider two training sets $S$ and $S'$ of size $n$ that differ at only one data point and we want to show how different the two saliency map losses will become if performing SGD algorithms on the two training sets. 

\begin{lemma}
\label{lemma:SGD}
Suppose the training loss function $\ell(f_{\mathcal W}(x),y)$ is $L$-Lipschitz, $\beta$-smooth for all $x,y$. 
If we run SGD for $T$ steps with a decaying step size $\alpha_t\le c/t$ in iteration $t$, then we can bound the stability error of $\mathrm{Sal}$ with $L'$-Lipschitz and $M'$-bounded corresponding saliency map loss by
\begin{align*}
\epsilon_{\mathrm{stability}}(\mathrm{Sal})\le \frac{1+\beta c}{n-1}(2cLL')^{\frac{1}{\beta c+1}}(TM')^{\frac{\beta c}{\beta c+1}}
\end{align*}  

Suppose the loss function is additionally convex and the step sizes satisfy $\alpha_t\le 2/\beta$, then we further have \begin{align*}
    \epsilon_{\mathrm{stability}}(\mathrm{Sal})\le \frac{2LL'}{n}\sum_{t=1}^T\alpha_t
\end{align*}
\end{lemma}

\begin{proof}
To derive a non-trivial upper bound, we note that typically it will take some time before the two SGD trajectories start to diverge. Using this idea, for any time step $t_0$, denote the event $A$ as whether $\delta_{t_0}=0$. If we randomly choose a sample in each iteration, we can show $P(A^c)=1-(1-1/n)^{t_0}\le \frac{t_0}{n}$. If we instead iterate the samples in the order of a random permutation, $P(A^c)\le \frac{t_0}{n}$.
Then we have \begin{align*}
    &\mathbb{E}[\ell'(w_T,x,y)-\ell'(w'_T,x,y)]\\
    =\,&P(A)\cdot\mathbb{E}[\ell'(w_T,x,y)-\ell'(w'_T,x,y)\mid A]+P(A^c)\cdot\mathbb{E}[\ell'(w_T,x,y)-\ell'(w'_T,x,y)\mid A^c]\\
    \le\,& \mathbb{E}[\ell'(w_T,x,y)-\ell'(w'_T,x,y)\mid A]+P(A^c)\cdot\sup_{\mathcal W,x,y}\ell'(\mathcal W,x,y)\\
    \le\,& L'\mathbb{E}[\delta_T\mid \delta_{t_0}=0]+\frac{t_0}{n}M'
\end{align*}

Define $\Delta_t=\mathbb{E}[\delta_t\mid \delta_{t_0}=0]$, we then want to bound $\Delta_{t+1}$ in terms of $\Delta_{t}$. With probability $1-1/n$, the chosen sample is the same so the iteration function is also the same. According to lemma \ref{lemma:expansion}, the difference of parameters will multiply by at most $1+\alpha_t\beta$ after this iteration. With probability $1/n$, we need to upper bound the difference between the two SGD updates, and the bound is \begin{align*}
    \Vert w_{t+1}-w'_{t+1}\Vert&\le \Vert w_{t}-w'_{t}\Vert+\Vert \alpha_t \ell(f_{x_1}(w_t),y_1)\Vert+\Vert \alpha_t \ell(f_{x_2}(w_t),y_2)\Vert\\
    &\le \Vert w_{t}-w'_{t}\Vert+2\alpha_tL
\end{align*}

Combining the two cases, we have \begin{align*}
    \Delta_{t+1}&\le (1-1/n)(1+\alpha_t\beta)\Delta_t+\frac 1n\Delta_t+\frac{2\alpha_tL}{n}\\
    &\le \left(1+(1-1/n)c\beta/t\right)\Delta_t+\frac{2cL}{tn}\\
    &\le \exp\left((1-1/n)c\beta/t\right)\Delta_t+\frac{2cL}{tn}
\end{align*}

According to calculations in \citep{hardt2016train}, we can bound \begin{align*}
    \Delta_{T}\le \frac{2L}{\beta(n-1)}\left(\frac{T}{t_0}\right)^{\beta c}
\end{align*}

Plugging this into our previous inequality for any $t_0$, we have \begin{align*}
    \mathbb{E}[\ell'(w_T,x,y)-\ell'(w'_T,x,y)]&\le \frac{t_0}{n}M'+\frac{2LL'}{\beta(n-1)}\left(\frac{T}{t_0}\right)^{\beta c}\\
    &\le M'\left(\frac{t_0}{n}+\frac{2LL'/M'}{\beta(n-1)}\left(\frac{T}{t_0}\right)^{\beta c}\right)\\
\end{align*}

When choosing $$t_0=(2cLL'/M')^{\frac{1}{\beta c+1}}T^{\frac{\beta c}{\beta c+1}},$$ we can get the following satisfying upper bound \begin{align*}
    \mathbb{E}[\ell'(w_T,x,y)-\ell'(w'_T,x,y)]&\le \frac{1+\beta c}{n-1}(2cLL'/M')^{\frac{1}{\beta c+1}}T^{\frac{\beta c}{\beta c+1}}M'\\
    &\le \frac{1+\beta c}{n-1}(2cLL')^{\frac{1}{\beta c+1}}(TM')^{\frac{\beta c}{\beta c+1}}
\end{align*}

Then we consider the case where the loss function is additionally convex and $\alpha_t\le 2/\beta$. Define $\delta_t=\Vert w_t-w'_t\Vert$ as the parameter difference after $t$ iterations, with $\delta_0=0$ originally. Then, we have $$\mathbb{E}[\ell'(w_t,x,y)-\ell'(w'_t,x,y)]\le L'\mathbb{E}[\delta_t]$$ and our goal would be to bound $\mathbb{E}[\delta_t]$.
We want to bound $\mathbb{E}[\delta_{t+1}]$ using $\mathbb{E}[\delta_{t}]$ to get an iterative formula. Consider the projected SGD update at time $t+1$. Since the projection operation does not increase Euclidean distance, we can bypass the operation when deriving an upper bound. With probability $1-1/n$, we choose the same data point to compute loss for our two parameters. In this case, the difference of two parameters will not increase due to lemma \ref{lemma:expansion} and the condition $\alpha_t\le 2/\beta$. With probability $1/n$ the chosen data points are different. In addition to the bound for the first case, we have an extra term for the difference between the two SGD updates. That is, $\Vert\alpha\nabla \ell(f_{x_1}(w_t),y_1)-\alpha\nabla \ell(f_{x_2}(w_t),y_2)\Vert\le 2\alpha_t L$.

Therefore, \begin{align*}
    \mathbb{E}[\delta_{t+1}]\le \mathbb{E}[\delta_{t}]+\frac{2\alpha_t L}{n}
\end{align*} 

Unfolding this formula iteratively, we have \begin{align*}
    \mathbb{E}[\delta_{t}]\le \frac{2L}{n}\sum_{t=1}^T\alpha_t.
\end{align*} Therefore, $\mathbb{E}[\ell'(w_t,x,y)-\ell'(w'_t,x,y)]\le \frac{2LL'}{n}\sum_{t=1}^T\alpha_t$ holds for all $S,S',x,y$ so the desired bound holds.

    
\end{proof}

\subsection{Proof of Theorem \ref{thm4.2}}
\label{thm4.2_proof}
\begin{lemma}
\label{lemma:simplegrad-bounded}
    For the class of $k$-layer neural networks $f_{\mathcal W}(x)=W_k\phi(W_{k-1}\phi(\dots\phi(W_1x)))\in \mathbb R^c$ with 1-Lipschitz, 1-smooth activation $\phi$ where $\phi(0) = 0$, we have $$\Vert \operatorname{SimpleGrad}_{A(S)}(x,y)\Vert\le \Delta_{k,0}$$
\end{lemma}
\begin{proof}
    By the definition of gradient, we have the iteration formula \begin{align*}
        \Vert \nabla_x f_{\mathcal W}^k(x)\Vert&=\Vert W_k\cdot \operatorname{diag}(\phi'(f_{\mathcal W}^{k-1}(x)))\cdot \nabla_x f_{\mathcal W}^{k-1}(x)\Vert\\
        &\le B_k\Vert \nabla_x f_{\mathcal W}^{k-1}(x)\Vert
    \end{align*}
    We can easily show that when $k=1$, $\Vert \nabla_x f_{\mathcal W}^1(x)\Vert=\Vert W_1\Vert\le B_1$. Thus solving the iterative formula gives $$\Vert \nabla_x f_{\mathcal W}^k(x)\Vert\le \prod_{i=1}^kB_i=\Delta_{k,0}$$
\end{proof}

\begin{lemma}
\label{lemma:simple}
    For Simple-Grad, the Lipschitz constant with respect to $\mathcal W$ is bounded by
    \begin{align*}
        \operatorname{lip}_{\mathcal W}(\operatorname{SimpleGrad}_{A(S)}(x,y))\le 2\Delta_{k,0}\sum_{i=1}^{k-1}\Delta_{i,1}C
    \end{align*}
\end{lemma}
\begin{proof}
Consider the difference at points $\mathcal W$ and $\mathcal W+\mathcal V$. \begin{align*}
    &\Vert \nabla_x f_{\mathcal W}(x)-\nabla_x f_{\mathcal W+\mathcal V}(x)\Vert\\
    \le\,&\Vert W_k\cdot \operatorname{diag}(\phi'(f_{\mathcal W}^{k-1}(x)))\nabla_x f_{\mathcal W}^{k-1}(x)-(W_k+V_k)\cdot \operatorname{diag}(\phi'(f_{\mathcal W+\mathcal V}^{k-1}(x)))\nabla_x f_{\mathcal W+\mathcal V}^{k-1}(x)\Vert\\
    \le\,&\Vert V_k\cdot \operatorname{diag}(\phi'(f_{\mathcal W}^{k-1}(x)))\nabla_x f_{\mathcal W}^{k-1}(x)\Vert\\
    &\:+\Vert (W_k+V_k)\cdot (\operatorname{diag}(\phi'(f_{\mathcal W}^{k-1}(x)))-\operatorname{diag}(\phi'(f_{\mathcal W+\mathcal V}^{k-1}(x))))\nabla_x f_{\mathcal W}^{k-1}(x)\Vert\\
    &\:+\Vert (W_k+V_k)\cdot \operatorname{diag}(\phi'(f_{\mathcal W+\mathcal V}^{k-1}(x)))(\nabla_x f_{\mathcal W}^{k-1}(x)-\nabla_x f_{\mathcal W+\mathcal V}^{k-1}(x))\Vert\\
    \le\,&\Vert\mathcal V\Vert\cdot\Vert\nabla_x f_{\mathcal W}^{k-1}(x)\Vert+B_k\cdot \operatorname{lip}_{\mathcal W}(f_{\mathcal W}^{k-1}(x))\Vert\mathcal V\Vert\cdot \Vert\nabla_x f_{\mathcal W}^{k-1}(x)\Vert+B_k \operatorname{lip}_{\mathcal W}(\nabla_x f_{\mathcal W}^{k-1}(x))\Vert\mathcal V\Vert
\end{align*}
Thus we have \begin{align*}
    \operatorname{lip}_{\mathcal W}(\nabla_x f_{\mathcal W}^{k}(x))&\le \Delta_{k-1,0}+B_k\Delta_{k-1,0}\Delta_{k-1,1}C+B_k\operatorname{lip}_{\mathcal W}(\nabla_x f_{\mathcal W}^{k-1}(x))\\
    &\le 2\Delta_{k,0}\Delta_{k-1,1}C+B_k\operatorname{lip}_{\mathcal W}(\nabla_x f_{\mathcal W}^{k-1}(x))\\
    &\le 2\Delta_{k,0}\sum_{i=1}^{k-1}\Delta_{i,1}C
\end{align*}
\end{proof}
Given the above two lemmas, we know the loss function defined by Simple-Grad \begin{align*}
    \ell'(A(S),x,y)=\Vert \operatorname{SimpleGrad}_{A(S)}(x,y)-\operatorname{SimpleGrad}_{A(D)}(x,y)\Vert
\end{align*} is $M'$-bounded and $L'$-Lipschitz, with $M'=2\Delta_{k,0}$ and $L'=2\Delta_{k,0}\sum_{i=1}^{k-1}\Delta_{i,1}C$. The first constant is due to triangular inequality and the second one is since \begin{align*}
    \Vert \ell'(\mathcal W,x,y)-\ell'(\mathcal W+\mathcal V,x,y)\Vert
    \le \Vert \nabla_x f_{\mathcal W}(x)-\nabla_x f_{\mathcal W+\mathcal V}(x)\Vert.
\end{align*}
Applying Lemma \ref{lemma:SGD} with the $M'$ and $L'$ then proves the theorem.
\hfill$\qedsymbol$

\subsection{Additional Stability Error Bound for Integrated-Grad}
We also provide the following theorem for Integrated-Grad. Even with the integration operation, we cannot give a stability error upper bound better than that of Simple-Grad, due to the element-wise multiplication operation.

\begin{theorem}
\label{thm4.3}
\textbf{(a) Convex Loss} Suppose Assumption \ref{assumption}(a), \ref{assumption}(b), and \ref{assumption}(c) hold. If we run the SGD algorithm with step sizes $\alpha_t\le 1/\beta$ for $T$ times, we can bound the stability error of Integrated-Grad by \begin{align*}
    \epsilon_{\mathrm{stability}}(\operatorname{IntegratedGrad})\le \operatorname{Up}_{\text{SC}}(\operatorname{SimpleGrad})\cdot 2C
\end{align*}

\textbf{(b) Non-Convex Loss} 
Suppose Assumption \ref{assumption}(b) and \ref{assumption}(c) hold. If we run SGD for $T$ steps with a decaying step size $\alpha_t\le c/t$ in iteration $t$, we can bound the stability error of Integrated-Grad by
\begin{align*}
\epsilon_{\mathrm{stability}}(\operatorname{IntegratedGrad})\le &\operatorname{Up}_{\text{NC}}(\operatorname{SimpleGrad})\cdot 2C
\end{align*}
\end{theorem}

\begin{lemma}
    For Integrated-Grad, the $\ell_2$-norm of the saliency map is bounded by
    \begin{align*}
        \Vert \operatorname{IntegratedGrad}_{A(S)}(x,y)\Vert\le 2\Delta_{k,0}C
    \end{align*}
\end{lemma}
\begin{proof}
Given that we can upper bound the norm of an integral by the integral of the norm, we have
    \begin{align*}
        \left\Vert (x-x_0)\odot \int_0^1 \nabla_x (f_{x_0+\alpha(x-x_0)}(\mathcal W))_y\mathrm{d}\alpha\right\Vert&\le \Vert x-x_0\Vert\cdot \int_0^1\Vert \nabla_x f_{x_0+\alpha(x-x_0)}(\mathcal W)\Vert \mathrm{d}\alpha\\
        &\le 2C\sup_{x'}\Vert \nabla_x f_{\mathcal W}(x')\Vert\\
        &\le 2\Delta_{k,0}C
    \end{align*}
\end{proof}

\begin{lemma}
\label{lemma:integrated}
    For Integrated-Grad, the Lipschitz constant with respect to $\mathcal W$ is bounded by
    \begin{align*}
        \operatorname{lip}_{\mathcal W}(\operatorname{IntegratedGrad}_{A(S)}(x,y))\le 4\Delta_{k,0}\sum_{i=1}^{k-1}\Delta_{i,1}C^2
    \end{align*}
\end{lemma}
\begin{proof}
    We have \begin{align*}
        &\left\Vert (x-x_0)\odot \int_0^1 \nabla_x (f_{x_0+\alpha(x-x_0)}(\mathcal W))_y\mathrm{d}\alpha-(x-x_0)\odot \int_0^1 \nabla_x (f_{x_0+\alpha(x-x_0)}(\mathcal W+\mathcal V))_y\mathrm{d}\alpha\right\Vert\\
        \le\, &\Vert x-x_0\Vert\cdot \left\Vert \int_0^1 (\nabla_x f_{x_0+\alpha(x-x_0)}(\mathcal W)-\nabla_x f_{x_0+\alpha(x-x_0)}(\mathcal W+\mathcal V))\mathrm{d}\alpha\right\Vert\\
        \le\, &2C\cdot \sup_{x'}\Vert \nabla_x f_{\mathcal W}(x')-\nabla_x f_{\mathcal W+\mathcal V}(x')\Vert\\
        \le\, &2C\cdot 2\Delta_{k,0}\sum_{i=1}^{k-1}\Delta_{i,1}C\Vert \mathcal V\Vert\\
        =\, &4\Delta_{k,0}\sum_{i=1}^{k-1}\Delta_{i,1}C^2\Vert \mathcal V\Vert,
    \end{align*} which finishes the proof.
\end{proof}
Given the above two lemmas, with a similar argument as in section \ref{thm4.2_proof}, we know the loss function defined by Integrated-Grad is $M'$-bounded and $L'$-Lipschitz, with $M'=4\Delta_{k,0}C$ and $L'=4\Delta_{k,0}\sum_{i=1}^{k-1}\Delta_{i,1}C^2$.
Applying Lemma \ref{lemma:SGD} then proves the theorem.

\hfill$\qedsymbol$

\subsection{Proof of Theorem \ref{thm4.4}}
For the proof of Smooth-Grad, we need Stein's lemma (lemma \ref{lemma:stein}) introduced earlier. In the following lemmas, . We also have the following lemma to bound the expected $\ell_2$-norm of a Gaussian vector.

\begin{lemma}
    The average norm of an $m$-dimensional vector given by a normal distribution is bounded by
    $$\mathbb{E}_{x\sim N(0,\sigma^2I)}\Vert x\Vert\le \sigma\sqrt m.$$
\end{lemma}

\begin{proof}
Note that $\Vert x\Vert^2$ is the sum of the squares of $m$ random variables with distributions $N(0,\sigma^2)$. Thus $$\mathbb{E}_{x\sim N(0,\sigma^2I)}\Vert x\Vert^2=m\mathbb{E}_{z\sim N(0,\sigma^2)}z^2=m\sigma^2.$$
According to Jensen's inequality, since $f(x)=\sqrt x$ is a concave function, we have $$\mathbb{E}_{x\sim N(0,\sigma^2I)}\Vert x\Vert=\mathbb{E}_{x\sim N(0,\sigma^2I)}\sqrt{\Vert x\Vert^2}\le \sqrt{\mathbb{E}_{x\sim N(0,\sigma^2I)}\Vert x\Vert^2}=\sigma\sqrt m.$$
\end{proof}

\begin{lemma}
    For Smooth-Grad, if we additionally normalize the perturbed input $x+z$ so that the norm of the neural network input does not exceed $C$, the $\ell_2$-norm of the saliency map is bounded by
    \begin{align*}
        \Vert \operatorname{SmoothGrad}_{A(S)}(x,y)\Vert\le \Delta_{k,0}
    \end{align*}
\end{lemma}
\begin{proof}
    Using the result in Lemma \ref{lemma:simplegrad-bounded}, we have
    \begin{align*}
        \Vert \mathbb{E}_{z\sim N(0,\sigma^2 I)}\nabla_xf_{\mathcal W}(x+z)\Vert \le \mathbb{E}_{z\sim N(0,\sigma^2 I)}\Vert \nabla_xf_{\mathcal W}(x+z)\Vert\le \Delta_{k,0}
    \end{align*}
\end{proof}

\begin{lemma}
\label{lemma:smooth}
    For Smooth-Grad, if we additionally normalize the perturbed input $x+z$ so that the norm of the neural network input does not exceed $C$, the Lipschitz constant with respect to $\mathcal W$ is bounded by
    \begin{align*}
        \operatorname{lip}_{\mathcal W}(\operatorname{SmoothGrad}_{A(S)}(x,y))\le \frac{1}{\sigma}C\Delta_{k,1}
    \end{align*}
\end{lemma}
\begin{proof}
    By applying Stein's lemma, we have \begin{align*}
        &\Vert \mathbb{E}_{z\sim N(0,\sigma^2 I)}\left[\nabla_x(f_{\mathcal W}(x+z))_y-\nabla_x(f_{\mathcal W+\mathcal V}(x+z))_y\right]\Vert\\
        =\, &\left\Vert \mathbb{E}_{z\sim N(0,\sigma^2 I)}\left[\frac{z}{\sigma^2}((f_{\mathcal W}(x+z))_y-(f_{\mathcal W+\mathcal V}(x+z))_y)\right]\right\Vert\\
        =\, &\max_{\Vert u\Vert=1}\mathbb{E}_{z\sim N(0,\sigma^2I)}\left[\frac{u^Tz}{\sigma^2}((f_{\mathcal W}(x+z))_y-(f_{\mathcal W+\mathcal V}(x+z))_y)\right]\\
        \le\, &\max_{\Vert u\Vert=1}\sqrt{\mathbb{E}_{z\sim N(0,\sigma^2I)}\left[\left(\frac{u^Tz}{\sigma^2}\right)^2\right]\mathbb{E}_{z\sim N(0,\sigma^2I)}\left[\left((f_{\mathcal W}(x+z))_y-(f_{\mathcal W+\mathcal V}(x+z))_y\right)^2\right]}\\
        \le\, &\frac{1}{\sigma}C\Delta_{k,1}\Vert\mathcal V\Vert
    \end{align*}
    Therefore, we have $$\operatorname{lip}_{\mathcal W}(\operatorname{SmoothGrad}_{A(S)}(x,y))\le \frac{1}{\sigma}C\Delta_{k,1}.$$
\end{proof}
Given the above two lemmas, with a similar argument as in section \ref{thm4.2_proof}, we know the loss function defined by Integrated-Grad is $M'$-bounded and $L'$-Lipschitz, with $M'=2\Delta_{k,0}$ and $L'=\frac{1}{\sigma}C\Delta_{k,1}$.
Applying Lemma \ref{lemma:SGD} then proves the theorem.
\hfill$\qedsymbol$

\subsection{Proof of Proposition \ref{thm5.1}}
Consider Simple-Grad. Denote $\mathcal W=A(S)$ for simplicity. For any $S,x,y,z$, we have \begin{align*}
    &\Vert\operatorname{SimpleGrad}_{A(S)}(x+z,y)-\operatorname{SimpleGrad}_{A(S)}(x,y)\Vert\\
    \le\, &\Vert \nabla_x f_{\mathcal W}^k(x)-\nabla_x f_{\mathcal W}^k(x+z)\Vert\\
    =\, &\Vert W_k\cdot \operatorname{diag}(\phi'(f_{\mathcal W}^{k-1}(x))\nabla_xf_{\mathcal W}^{k-1}(x)-W_k\cdot \operatorname{diag}(\phi'(f_{\mathcal W}^{k-1}(x+z))\nabla_xf_{\mathcal W}^{k-1}(x+z)\Vert\\
    \le\, &B_k(\Vert (\operatorname{diag}(\phi'(f_{\mathcal W}^{k-1}(x))-\operatorname{diag}(\phi'(f_{\mathcal W}^{k-1}(x+z)))\nabla_xf_{\mathcal W}^{k-1}(x)\Vert\\
    &\: +\Vert \operatorname{diag}(\phi'(f_{\mathcal W}^{k-1}(x+z))(\nabla_xf_{\mathcal W}^{k-1}(x)-\nabla_xf_{\mathcal W}^{k-1}(x+z))\Vert)\\
    \le\, &B_k(\Vert f_{\mathcal W}^{k-1}(x)-f_{\mathcal W}^{k-1}(x+z)\Vert \cdot \Vert\nabla_xf_{\mathcal W}^{k-1}(x)\Vert+\Vert \nabla_xf_{\mathcal W}^{k-1}(x)-\nabla_xf_{\mathcal W}^{k-1}(x+z)\Vert)\\
    \le\, &B_k(\Delta_{k-1,0}^2\Vert z\Vert+\Vert \nabla_xf_{\mathcal W}^{k-1}(x)-\nabla_xf_{\mathcal W}^{k-1}(x+z)\Vert)\\
    \le\, &\sum_{i=1}^{k-1}\prod_{j=i+1}^k B_j\Delta_{i,0}^2\Vert z\Vert\\
    =\, &\Delta_{k,0}\sum_{i=1}^{k-1}\Delta_{i,0}\cdot \Vert z\Vert
\end{align*}

Then we have for any $S,x$, \begin{align*}
    &\mathbb{E}_{z\sim N(0,\sigma^2 I)}\Vert\operatorname{SimpleGrad}_{A(S)}(x+z,y)-\operatorname{SimpleGrad}_{A(S)}(x,y)\Vert\\
    \le\, & \mathbb{E}_{z\sim N(0,\sigma^2 I)}\left[\Delta_{k,0}\sum_{i=1}^{k-1}\Delta_{i,0}\cdot \Vert z\Vert\right] \\
    \le\, & \Delta_{k,0}\sum_{i=1}^{k-1}\Delta_{i,0}\sigma \sqrt m,
\end{align*} which implies \begin{align*}
    \epsilon_{\mathrm{fidelity}}(\operatorname{SimpleGrad})&\le \Delta_{k,0}\sum_{i=1}^{k-1}\Delta_{i,0}\sigma\sqrt{m}.
\end{align*}

Consider Integrated-Grad. For any $S,x,y,z$, we have \begin{align*}
    &\Vert\operatorname{IntegratedGrad}_{A(S)}(x+z,y)-\operatorname{IntegratedGrad}_{A(S)}(x,y)\Vert\\
    =\, &\left\Vert (x+z-x_0)\odot \int_0^1 \nabla_x (f_{\mathcal W}(x_0+\alpha(x+z-x_0)))_y\mathrm{d}\alpha-(x-x_0)\odot \int_0^1 \nabla_x (f_{\mathcal W}(x_0+\alpha(x-x_0)))_y\mathrm{d}\alpha\right\Vert\\
    \le\, &\left\Vert (x+z-x_0)\odot \left(\int_0^1 \nabla_x (f_{\mathcal W}(x_0+\alpha(x+z-x_0)))_y\mathrm{d}\alpha- \int_0^1 \nabla_x (f_{\mathcal W}(x_0+\alpha(x-x_0)))_y\mathrm{d}\alpha\right)\right\Vert\\
    &\: +\left\Vert(x+z-x_0)\odot \int_0^1 \nabla_x (f_{\mathcal W}(x_0+\alpha(x-x_0)))_y\mathrm{d}\alpha-(x-x_0)\odot \int_0^1 \nabla_x (f_{\mathcal W}(x_0+\alpha(x-x_0)))_y\mathrm{d}\alpha\right\Vert\\
    \le\, &\Vert x+z-x_0\Vert\sup_{\alpha\in [0,1]}\Vert \nabla_x f_{\mathcal W}(x_0+\alpha(x+z-x_0))-\nabla_x f_{\mathcal W}(x_0+\alpha(x-x_0))\Vert\\
    &\:+\Vert z\Vert\sup_{\alpha\in [0,1]}\Vert \nabla_x f_{\mathcal W}(x_0+\alpha(x-x_0))\Vert\\
    \le\, &\Vert x+z-x_0\Vert\cdot \Delta_{k,0}\sum_{i=1}^{k-1}\Delta_{i,0}\min(\Vert z\Vert,C)+\Vert z\Vert\Delta_{k,0}
\end{align*}

Then we have for any $S,x,y$, \begin{align*}
    &\mathbb{E}_{z\sim N(0,\sigma^2 I)}\Vert\operatorname{IntegratedGrad}_{A(S)}(x+z,y)-\operatorname{IntegratedGrad}_{A(S)}(x,y)\Vert\\
    \le\, &\Delta_{k,0}\sum_{i=1}^{k-1}\Delta_{i,0}\mathbb{E}_{z\sim N(0,\sigma^2 I)}[\Vert x-x_0\Vert\Vert z\Vert+\Vert z\Vert C]+\mathbb{E}_{z\sim N(0,\sigma^2 I)}[\Vert z \Vert\Delta_{k,0}]\\
    \le\, &\Delta_{k,0}\sum_{i=1}^{k-1}\Delta_{i,0}\cdot 3C\sigma\sqrt m+\Delta_{k,0}\sigma\sqrt m\\
    \le\, &\Delta_{k,0}\sum_{i=1}^{k-1}\Delta_{i,0}\sigma\sqrt m\left(3C+\frac{1}{\sum_{i=1}^{k-1}\Delta_{i,0}}\right)\\
    \le\, &\Delta_{k,0}\sum_{i=1}^{k-1}\Delta_{i,0}\sigma\sqrt m(3C+1),
\end{align*} which concludes the second part of our desired theorem, which is \begin{align*}
    \epsilon_{\mathrm{fidelity}}(\operatorname{IntegratedGrad})&\le \Delta_{k,0}\sum_{i=1}^{k-1}\Delta_{i,0}(3C+1)\sigma\sqrt{m}.
\end{align*}

\hfill$\qedsymbol$

%% file: AISTATS_2025.bbl
\begin{thebibliography}{}

\bibitem[Adebayo et~al., 2018]{adebayo2018sanity}
Adebayo, J., Gilmer, J., Muelly, M., Goodfellow, I., Hardt, M., and Kim, B.
  (2018).
\newblock Sanity checks for saliency maps.
\newblock {\em Advances in neural information processing systems}, 31.

\bibitem[Arun et~al., 2021]{arun2021assessing}
Arun, N., Gaw, N., Singh, P., Chang, K., Aggarwal, M., Chen, B., Hoebel, K.,
  Gupta, S., Patel, J., Gidwani, M., et~al. (2021).
\newblock Assessing the trustworthiness of saliency maps for localizing
  abnormalities in medical imaging.
\newblock {\em Radiology: Artificial Intelligence}, 3(6):e200267.

\bibitem[Baehrens et~al., 2010]{baehrens2010explain}
Baehrens, D., Schroeter, T., Harmeling, S., Kawanabe, M., Hansen, K., and
  M{\"u}ller, K.-R. (2010).
\newblock How to explain individual classification decisions.
\newblock {\em The Journal of Machine Learning Research}, 11:1803--1831.

\bibitem[Bien et~al., 2018]{bien2018deep}
Bien, N., Rajpurkar, P., Ball, R.~L., Irvin, J., Park, A., Jones, E., Bereket,
  M., Patel, B.~N., Yeom, K.~W., Shpanskaya, K., et~al. (2018).
\newblock Deep-learning-assisted diagnosis for knee magnetic resonance imaging:
  development and retrospective validation of mrnet.
\newblock {\em PLoS medicine}, 15(11):e1002699.

\bibitem[Bousquet and Elisseeff, 2002]{bousquet2002stability}
Bousquet, O. and Elisseeff, A. (2002).
\newblock Stability and generalization.
\newblock {\em The Journal of Machine Learning Research}, 2:499--526.

\bibitem[Chattopadhay et~al., 2018]{chattopadhay2018generalized}
Chattopadhay, A., Sarkar, A., Howlader, P., and Balasubramanian, V.~N. (2018).
\newblock Grad-cam++: Generalized gradient-based visual explanations for deep
  convolutional networks.
\newblock In {\em 2018 IEEE winter conference on applications of computer
  vision (WACV)}, pages 839--847. IEEE.

\bibitem[Deng et~al., 2009]{5206848}
Deng, J., Dong, W., Socher, R., Li, L.-J., Li, K., and Fei-Fei, L. (2009).
\newblock Imagenet: A large-scale hierarchical image database.
\newblock In {\em 2009 IEEE Conference on Computer Vision and Pattern
  Recognition}, pages 248--255.

\bibitem[Fel et~al., 2022]{fel2022good}
Fel, T., Vigouroux, D., Cad{\`e}ne, R., and Serre, T. (2022).
\newblock How good is your explanation? algorithmic stability measures to
  assess the quality of explanations for deep neural networks.
\newblock In {\em Proceedings of the IEEE/CVF Winter Conference on Applications
  of Computer Vision}, pages 720--730.

\bibitem[Freiesleben et~al., 2022]{freiesleben2022scientific}
Freiesleben, T., K{\"o}nig, G., Molnar, C., and Tejero-Cantero, A. (2022).
\newblock Scientific inference with interpretable machine learning: Analyzing
  models to learn about real-world phenomena.
\newblock {\em arXiv preprint arXiv:2206.05487}.

\bibitem[Ghorbani et~al., 2019]{ghorbani2019interpretation}
Ghorbani, A., Abid, A., and Zou, J. (2019).
\newblock Interpretation of neural networks is fragile.
\newblock In {\em Proceedings of the AAAI conference on artificial
  intelligence}, volume~33, pages 3681--3688.

\bibitem[Gildenblat and contributors, 2021]{jacobgilpytorchcam}
Gildenblat, J. and contributors (2021).
\newblock Pytorch library for cam methods.
\newblock \url{https://github.com/jacobgil/pytorch-grad-cam}.

\bibitem[Gong et~al., 2024a]{gong2024structured}
Gong, S., Dou, Q., and Farnia, F. (2024a).
\newblock Structured gradient-based interpretations via norm-regularized
  adversarial training.
\newblock In {\em Proceedings of the IEEE/CVF Conference on Computer Vision and
  Pattern Recognition}, pages 11009--11018.

\bibitem[Gong et~al., 2025]{gong2025boosting}
Gong, S., Haoyu, L., Dou, Q., and Farnia, F. (2025).
\newblock Boosting the visual interpretability of clip via adversarial
  fine-tuning.
\newblock In {\em The Thirteenth International Conference on Learning
  Representations}.

\bibitem[Gong et~al., 2024b]{gong2024super}
Gong, S., Zhang, J., Dou, Q., and Farnia, F. (2024b).
\newblock A super-pixel-based approach to the stable interpretation of neural
  networks.
\newblock {\em arXiv preprint arXiv:2412.14509}.

\bibitem[Graves et~al., 2013]{graves2013speech}
Graves, A., Mohamed, A.-r., and Hinton, G. (2013).
\newblock Speech recognition with deep recurrent neural networks.
\newblock In {\em 2013 IEEE international conference on acoustics, speech and
  signal processing}, pages 6645--6649. Ieee.

\bibitem[Hardt et~al., 2016]{hardt2016train}
Hardt, M., Recht, B., and Singer, Y. (2016).
\newblock Train faster, generalize better: Stability of stochastic gradient
  descent.
\newblock In {\em International conference on machine learning}, pages
  1225--1234. PMLR.

\bibitem[He et~al., 2015]{he2015deep}
He, K., Zhang, X., Ren, S., and Sun, J. (2015).
\newblock Deep residual learning for image recognition.

\bibitem[Ismail et~al., 2021]{ismail2021improving}
Ismail, A.~A., Corrada~Bravo, H., and Feizi, S. (2021).
\newblock Improving deep learning interpretability by saliency guided training.
\newblock {\em Advances in Neural Information Processing Systems},
  34:26726--26739.

\bibitem[Kim et~al., 2019]{kim2019saliency}
Kim, B., Seo, J., Jeon, S., Koo, J., Choe, J., and Jeon, T. (2019).
\newblock Why are saliency maps noisy? cause of and solution to noisy saliency
  maps.
\newblock In {\em 2019 IEEE/CVF International Conference on Computer Vision
  Workshop (ICCVW)}, pages 4149--4157. IEEE.

\bibitem[Kindermans et~al., 2019]{kindermans2019reliability}
Kindermans, P.-J., Hooker, S., Adebayo, J., Alber, M., Sch{\"u}tt, K.~T.,
  D{\"a}hne, S., Erhan, D., and Kim, B. (2019).
\newblock The (un) reliability of saliency methods.
\newblock {\em Explainable AI: Interpreting, explaining and visualizing deep
  learning}, pages 267--280.

\bibitem[Krizhevsky and Hinton, 2009]{Krizhevsky09}
Krizhevsky, A. and Hinton, G. (2009).
\newblock Learning multiple layers of features from tiny images.
\newblock {\em Master's thesis, Department of Computer Science, University of
  Toronto}.

\bibitem[Krizhevsky et~al., 2012]{krizhevsky2012imagenet}
Krizhevsky, A., Sutskever, I., and Hinton, G.~E. (2012).
\newblock Imagenet classification with deep convolutional neural networks.
\newblock {\em Advances in neural information processing systems}, 25.

\bibitem[Levine et~al., 2019]{levine2019certifiably}
Levine, A., Singla, S., and Feizi, S. (2019).
\newblock Certifiably robust interpretation in deep learning.
\newblock {\em arXiv preprint arXiv:1905.12105}.

\bibitem[Lin et~al., 2019]{lin2019steins}
Lin, W., Khan, M.~E., and Schmidt, M. (2019).
\newblock Stein's lemma for the reparameterization trick with exponential
  family mixtures.

\bibitem[Liu et~al., 2021]{liu2021Swin}
Liu, Z., Lin, Y., Cao, Y., Hu, H., Wei, Y., Zhang, Z., Lin, S., and Guo, B.
  (2021).
\newblock Swin transformer: Hierarchical vision transformer using shifted
  windows.
\newblock In {\em Proceedings of the IEEE/CVF International Conference on
  Computer Vision (ICCV)}.

\bibitem[Liu et~al., 2022]{liu2022convnet}
Liu, Z., Mao, H., Wu, C.-Y., Feichtenhofer, C., Darrell, T., and Xie, S.
  (2022).
\newblock A convnet for the 2020s.
\newblock {\em Proceedings of the IEEE/CVF Conference on Computer Vision and
  Pattern Recognition (CVPR)}.

\bibitem[maintainers and contributors, 2016]{torchvision2016}
maintainers, T. and contributors (2016).
\newblock Torchvision: Pytorch's computer vision library.
\newblock \url{https://github.com/pytorch/vision}.

\bibitem[Minaee et~al., 2021]{minaee2021deep}
Minaee, S., Kalchbrenner, N., Cambria, E., Nikzad, N., Chenaghlu, M., and Gao,
  J. (2021).
\newblock Deep learning--based text classification: a comprehensive review.
\newblock {\em ACM computing surveys (CSUR)}, 54(3):1--40.

\bibitem[Mitani et~al., 2020]{mitani2020detection}
Mitani, A., Huang, A., Venugopalan, S., Corrado, G.~S., Peng, L., Webster,
  D.~R., Hammel, N., Liu, Y., and Varadarajan, A.~V. (2020).
\newblock Detection of anaemia from retinal fundus images via deep learning.
\newblock {\em Nature Biomedical Engineering}, 4(1):18--27.

\bibitem[Nesterov, 2014]{10.5555/2670022}
Nesterov, Y. (2014).
\newblock {\em Introductory Lectures on Convex Optimization: A Basic Course}.
\newblock Springer Publishing Company, Incorporated, 1 edition.

\bibitem[Selvaraju et~al., 2017]{selvaraju2017grad}
Selvaraju, R.~R., Cogswell, M., Das, A., Vedantam, R., Parikh, D., and Batra,
  D. (2017).
\newblock Grad-cam: Visual explanations from deep networks via gradient-based
  localization.
\newblock In {\em Proceedings of the IEEE international conference on computer
  vision}, pages 618--626.

\bibitem[Shrikumar et~al., 2017]{shrikumar2017learning}
Shrikumar, A., Greenside, P., and Kundaje, A. (2017).
\newblock Learning important features through propagating activation
  differences.
\newblock In {\em International conference on machine learning}, pages
  3145--3153. PMLR.

\bibitem[Simonyan et~al., 2013]{simonyan2013deep}
Simonyan, K., Vedaldi, A., and Zisserman, A. (2013).
\newblock Deep inside convolutional networks: Visualising image classification
  models and saliency maps.
\newblock {\em arXiv preprint arXiv:1312.6034}.

\bibitem[Smilkov et~al., 2017]{smilkov2017smoothgrad}
Smilkov, D., Thorat, N., Kim, B., Vi{\'e}gas, F., and Wattenberg, M. (2017).
\newblock Smoothgrad: removing noise by adding noise.
\newblock {\em arXiv preprint arXiv:1706.03825}.

\bibitem[Springenberg et~al., 2014]{springenberg2014striving}
Springenberg, J.~T., Dosovitskiy, A., Brox, T., and Riedmiller, M. (2014).
\newblock Striving for simplicity: The all convolutional net.
\newblock {\em arXiv preprint arXiv:1412.6806}.

\bibitem[Sundararajan et~al., 2017]{sundararajan2017axiomatic}
Sundararajan, M., Taly, A., and Yan, Q. (2017).
\newblock Axiomatic attribution for deep networks.
\newblock In {\em International conference on machine learning}, pages
  3319--3328. PMLR.

\bibitem[Wang et~al., 2004]{wang2004image}
Wang, Z., Bovik, A.~C., Sheikh, H.~R., and Simoncelli, E.~P. (2004).
\newblock Image quality assessment: from error visibility to structural
  similarity.
\newblock {\em IEEE transactions on image processing}, 13(4):600--612.

\bibitem[Woerl et~al., 2023]{woerl2023initialization}
Woerl, A.-C., Disselhoff, J., and Wand, M. (2023).
\newblock Initialization noise in image gradients and saliency maps.
\newblock In {\em Proceedings of the IEEE/CVF Conference on Computer Vision and
  Pattern Recognition}, pages 1766--1775.

\bibitem[Zeiler and Fergus, 2014]{zeiler2014visualizing}
Zeiler, M.~D. and Fergus, R. (2014).
\newblock Visualizing and understanding convolutional networks.
\newblock In {\em Computer Vision--ECCV 2014: 13th European Conference, Zurich,
  Switzerland, September 6-12, 2014, Proceedings, Part I 13}, pages 818--833.
  Springer.

\bibitem[Zhang and Farnia, 2023]{zhang2023moreaugrad}
Zhang, J. and Farnia, F. (2023).
\newblock Moreaugrad: Sparse and robust interpretation of neural networks via
  moreau envelope.
\newblock {\em arXiv preprint arXiv:2302.05294}.

\end{thebibliography}
